\documentclass{article} 

\usepackage{ssArxiv}
\usepackage[utf8]{inputenc}
\usepackage[title]{appendix}
\usepackage[dvipsnames]{xcolor}
\usepackage[
  colorlinks=true,
  citecolor=NavyBlue,
  linkcolor=Fuchsia,
  urlcolor=BrickRed
]{hyperref}

\usepackage{amsmath, amssymb}

\usepackage{alltt}
\usepackage{subcaption}
\usepackage{wrapfig}
\usepackage{booktabs}
\usepackage{bookmark}
\usepackage{nicefrac}
\usepackage{pifont}
\usepackage{tcolorbox}

\usepackage{tikz}

\usepackage{url}
\usepackage{enumitem}
\setlist[itemize]{
  noitemsep, topsep=0pt, label=$\blacktriangleright$, leftmargin=*
}
\usepackage{xspace}
\usepackage{framed}

\newcommand{\fstar}{f^\star}
\newcommand{\fhat}{\hat{f}}
\newcommand{\fbar}{\bar{f}}
\newcommand{\ftilde}{\tilde{f}}
\newcommand{\gin}{\mathrm{g}}
\newcommand{\ginhat}{\hat{\gin}}
\newcommand{\gintilde}{\tilde{\gin}}

\newcommand{\F}{\mathcal{F}}

\newcommand{\E}{\mathbb{E}}
\newcommand{\Real}{\mathbb{R}}

\newcommand{\er}{\mathcal{E}}
\newcommand{\eapp}{\mathcal{E}_{\text{{\footnotesize {\tt app}}}}}
\newcommand{\eest}{\mathcal{E}_{\text{{\footnotesize {\tt est}}}}}

\newcommand{\eopt}{\mathcal{E}_{\text{{\footnotesize {\tt opt}}}}}
\newcommand{\emcm}{\mathcal{E}_{\text{{\footnotesize {\tt mcm}}}}} 
\newcommand{\ehin}{\mathcal{E}_{\text{{\footnotesize {\tt hin}}}}} 

\newcommand{\rhoin}{\rho_{\text{{\footnotesize \tt in}}}}
\newcommand{\rhoout}{\rho_{\text{{\footnotesize \tt out}}}}

\DeclareMathOperator*{\argmin}{arg\,min}

\newcommand{\ua}{\textcolor{MidnightBlue}{$\boldsymbol{\uparrow}$}}
\newcommand{\da}{\textcolor{BurntOrange}{$\boldsymbol{\downarrow}$}}

\newcommand*\circled[1]{\tikz[baseline=(char.base)]{
    \node[shape=circle,draw,inner sep=1.25pt] (char) {#1};}}

\usepackage{amsthm}

\newtheorem{theorem}{Theorem}[section]

\newtheorem{heuristic}{Heuristic}[section]

\usepackage{natbib}

\title{%
  Toward Theoretical Guidance for Two Common Questions in 
  Practical Cross-Validation based Hyperparameter Selection%
}
\author{%
\name Parikshit Ram \email{p.ram@acm.org} \\
\addr{IBM Thomas J. Watson Research Center, Yorktown Heights, NY, USA}\\
\name Alexander G. Gray \email{alexander.gray@ibm.com}\\
\addr{IBM Thomas J. Watson Research Center, Yorktown Heights, NY, USA}\\
\name Horst C. Samulowitz \email{samulowitz@us.ibm.com}\\
\addr{IBM Thomas J. Watson Research Center, Yorktown Heights, NY, USA}\\
\name Gregory Bramble \email{bramble@us.ibm.com}\\
\addr{IBM Thomas J. Watson Research Center, Yorktown Heights, NY, USA}\\
}

\begin{document}

\maketitle

\begin{abstract}
We show, to our knowledge, the first theoretical treatments of two common
questions in cross-validation based hyperparameter selection: \circled{1}~After
selecting the best hyperparameter using a held-out set, we train the final model
using {\em all} of the training data -- since this may or may not improve future
generalization error, should one do this?  \circled{2}~During optimization such
as via SGD (stochastic gradient descent), we must set the optimization tolerance
$\rho$ -- since it trades off predictive accuracy with computation cost, how
should one set it?  Toward these problems, we introduce the {\em hold-in risk}
(the error due to not using the whole training data), and the {\em model class
  mis-specification risk} (the error due to having chosen the wrong model class)
in a theoretical view which is simple, general, and suggests heuristics that can
be used when faced with a dataset instance.  In proof-of-concept studies in
synthetic data where theoretical quantities can be controlled, we show that
these heuristics can, respectively, \circled{1}~always perform at least as well
as always performing retraining or never performing retraining,
\circled{2}~either improve performance or reduce computational overhead by
$2\times$ with no loss in predictive performance.
\end{abstract}
\section{Introduction}\label{sec:intro}
The learning process has various sources of errors. The first step in
(supervised) learning is the acquisition of (training) data. Given data, we
choose a model or function class $\F$ which corresponds to not just a {\em
  method} (such as Support Vector Machines, Generalized Linear Models, Neural
Networks, Decision trees) but their specific configuration governed by their
respective {\em hyperparameters} (such as regularization forms and penalties,
trees depth) -- these hyperparameters refer to anything that would affect the
predictive performance of the model learned from the training data. Given our
choice of the function class, the learning process searches for the function
that (approximately) minimizes the empirical risk (or some surrogate of it which
better represents the true risk or is easier to optimize). We currently have an
understanding of the factors~\citep{vapnik2006estimation,
  devroye2013probabilistic, bottou2008tradeoffs} affecting the {\em excess risk}
of this chosen function -- (i)~the choice of the function class and its capacity
to model the data generating process, (ii)~the use of an empirical risk {\em
  estimate} instead of the true risk, and (iii)~the approximation in the
empirical risk minimization (ERM).

However, in practice, the learning process is not limited to these steps. A
significant part of the whole exercise is the choice of the function class $\F$
(method and its specifications). Usually, we consider a (possibly large) set of
function classes and select one of them based on the data-driven process of
\emph{model selection} or {\em hyperparameter selection}. This search can be
done via grid search (searching over a discretized grid of hyperparameter
values) or random search~\citep{bergstra2012random}. However, AutoML (automated
machine learning) has spurred a lot of research in the area of {\em
  hyperparameter optimization} or HPO~\citep{smac, shahriari2016taking,
  snoek2012practical, bergstra2011algorithms, bergstra2013making}. The
automation allows us to look at even larger sets of function classes for
improved performance while being significantly more efficient than grid search
and more accurate than random search. The problem of HPO has been extended from
machine learning (ML) model configurations to the design of complete ML
pipelines known as the Combined Algorithm Selection and HPO (or CASH) problem,
with various schemes that handle (i)~pipelines with fixed
architecture~\citep{smac, bergstra2011algorithms, rakotoarison2019automated,
  liu2020admm, kishimoto2022bandit}, (ii)~searching over the pipeline
architectures~\citep{katz2020exploring, marinescu2021searching},
(iii)~deployment and fairness constraints~\citep{liu2020admm, ram2020solving},
and (iv)~operating in the decentralized setting~\citep{zhou2021flora,
  zhou2022single}, leading to multiple open-source tools~\citep{autoweka1,
  autoweka2, feurer2015efficient, feurer2020auto, komer2014hyperopt,
  bergstra2015hyperopt, baudart2020lale, baudart2021pipeline,
  hirzel2022gradual}.

There has also been a significant amount of theoretical work on development of
data-dependent penalties for penalty-based model selection, resulting in
guarantees in the form of ``oracle inequalities'' -- the expected excess risk of
the selected model can be shown to be within a multiplicative and additive
factor of the best possible excess risk if an oracle provided us with the best
hyperparameter. This has been widely studied in (binary)
classification~\citep{boucheron2005theory}, (bounded) regression and density
estimation~\citep{massart2007concentration, arlot2010survey}. However, in
practice, penalty-based model selection is not used for data-driven
hyperparameter selection, and we resort to some form of {\em
  cross-validation}\footnotemark (CV). These universally applicable CV
techniques have been shown to be theoretically competitive to the penalty-based
schemes at the cost of having less data for the learning since some amount of
data is ``held-out'' from the training data for validation
purposes~\citep{boucheron2005theory, arlot2010survey}. We focus on these
universal CV based HPO.%
\footnotetext{Existing literature terms the single training/validation split as
  ``cross-validation'' (see for example \citep{kearns1996bound, blum1999beating})
  and when there are multiple folds, it is specifically termed as ``$k$-fold
  cross-validation''.}
\paragraph{Our contributions.}
While CV based model selection has been studied theoretically, there are various
questions in practical HPO, which have not been explored in literature. A common
practice is the learning of the final model on the selected hyperparameter with
all available data, reintroducing the held-out data in the training. This is
standard practice in many commercial ML tools since user data is too precious to
not include in the training of the deployed model. In this paper, we provide
answers for the following questions\footnote{We presented a preliminary version
  of this work at the AutoML@ICML'21 workshop~\citep{ram2021leveraging}.}:
\begin{tcolorbox}[before skip=0.05cm, after skip=0.05cm, boxsep=0.0cm, middle=0.1cm]
\hypertarget{q1a}{\textsf{[Q1a]}}~\emph{Theoretically, when does this practice
  help in terms of excess risk and by how much?}
\tcblower
\hypertarget{q1b}{\textsf{[Q1b]}}~\emph{Is there a practical data-driven way of
  deciding whether to do this or not?}
\end{tcolorbox}
A common practice is to approximate the ERM with a large tolerance during the
HPO and perform a more accurate ERM during the final model training on the
selected HP to reduce the overall computational costs. We study the following
related questions:
\begin{tcolorbox}[before skip=0.05cm, after skip=0.05cm, boxsep=0.0cm, middle=0.1cm]
\hypertarget{q2a}{\textsf{[Q2a]}}~\emph{How does this ERM approximation in
  hyperparameter selection affect the excess risk?}
\tcblower
\hypertarget{q2b}{\textsf{[Q2b]}}~\emph{Is there a theoretically informed
  practical way of setting these approximation levels during the hyperparameter
  selection to better control the excess risk vs computation tradeoff?}
\end{tcolorbox}
Note that these answers can be utilized both by humans and by automated data
science systems.
\paragraph{Practical motivations.}
In applications with large amount of data, the importance of question
\hyperlink{q1a}{\textsf{Q1a}} (and consequently \hyperlink{q1b}{\textsf{Q1b}})
might appear minimal (the computational aspect of questions
\hyperlink{q2a}{\textsf{Q2a}} and \hyperlink{q2b}{\textsf{Q2b}} make them
critical with large data). However, we believe that even in such cases, the
important signal can still be quite small, and the relevant information in the
held-out set can have large impact ({\em positive or negative}) on the final
deployed model. For example, in various data science applications in finance,
the class/target of interest is usually very small. In the
\href{https://www.kaggle.com/c/home-credit-default-risk}{Home Credit Default
  Risk Kaggle Classification Challenge}, the training data has only around
$8.5\%$ of some $3\times10^5$ training examples from the positive class, and in
the \href{https://www.kaggle.com/c/ClaimPredictionChallenge}{AllState Claim
  Prediction Kaggle Regression Challenge}, the training set has less than $1\%$
of some $13\times 10^6$ examples with nonzero regression targets; rest are zero.
Furthermore, there are situations where the population has groups and the
underrepresented minority groups are significantly smaller than the majority
group(s). In this case, the presence or absence of the held-out data can have
significant (again positive or negative) impact on the fairness and accuracy of
the deployed model since many of the fairness metrics quantify the parity in the
group specific predictive performance.
Finally, ``small'' differences in predictive performance can have significant
impact depending on problems and applications -- in data science leaderboards
such as Kaggle, minor changes in final performance can lead to significant
reordering of the leaderboard (although, relevance of such competitions and
results to practical problems remain an open question).
Such scenarios motivate us to study the question of {\em whether} the practice
of retraining after reintroducing the data held-out during HPO is helpful
(\hyperlink{q1a}{\textsf{Q1a}}).
\paragraph{Empirical motivation.}
As evidence of the lack of clarity on this {\em whether} question, we consider
HPO for LightGBM~\citep{ke2017lightgbm} on 40 OpenML~\citep{OpenML2013} data sets
(we detail the evaluation setup in \S \ref{sec:probset:emp}). Of the 400
different HPO problems we solve, retraining has (i)~no {\em significant} effect
(positive or negative) in 51/400 (12.75\%) cases, (ii)~has a positive effect
(improving test error) in 260/400 (65\%) cases, and (iii)~has a negative effect
in 89/400 (22.25\%) cases. This highlights that {\em there is no single
  universal correct answer to this question}.
Hence, we study this {\em whether} question rigorously and provide a
theoretically motivated data-driven heuristic as one answer.
\paragraph{Outline.}
We present our precise problem setting, and existing \& novel excess risk
decompositions with empirical support in \S \ref{sec:probset}. Then we present
and evaluate our theoretical results, tradeoffs and practical heuristics for HPO
in \S \ref{sec:newform}.
We position our contributions against existing literature in \S
\ref{sec:litreview}, and conclude in \S \ref{sec:conc}.%
\begin{figure*}[ht]
\centering
\begin{subfigure}{0.3\textwidth}
  \centering
  \begin{tikzpicture}[scale=0.9]
  \draw[line width=1pt] (2, 0) .. controls (0, 1) and (0, 3) .. (3, 4);
  \draw[line width=1pt] (3, 4) .. controls (4, 3) and (4, 1) .. (2, 0) node[pos=0.9, right] {$\F_\lambda$};
  \fill[black] (0, 4) circle (4pt) node[right] {\ \  $\fstar$};
  \fill[blue] (1.3, 2.0) circle (4pt) node[below] {$\fbar_{\lambda}$};
  \fill[brown] (2.7, 2.6) circle (4pt) node[above] {$\fhat_{n,\lambda}$};
  \fill[magenta] (3.0, 1.2) circle (4pt) node[left] {$\ftilde_{n,\lambda}$};
  \draw[<->, gray, thick] (0.1, 3.9) -- (2.9, 1.3) node[pos=0.3, above] {$\mathcal{E}$};
  \draw[<->, orange, dashed, thick] (0.1, 3.9) -- (1.2, 2.1) node[pos=0.5, left] {$\eapp$};
  \draw[<->, red!50, dashed, thick] (2.7, 2.45) -- (3, 1.3) node[pos=0.3, right] {$\eopt$};
  \draw[<->, violet, dashed, thick] (1.4,2.1) -- (2.6, 2.5) node[pos=0.6, above] {$\eest$};
  \end{tikzpicture}
    \caption{Original decomposition}
  \label{fig:ed-orig}
\end{subfigure}
~
\begin{subfigure}{0.3\textwidth}
  \centering
  \begin{tikzpicture}[scale=0.9]
  \fill[black] (0, 4) circle (4pt) node[right] {\ \ $\fstar$};
  \draw[line width=1pt] (1, 2.5) .. controls (0, 3) .. (1, 4);
  \draw[line width=1pt] (1, 4) .. controls (1.5, 3) .. (1, 2.5) node[pos=0.5, right] {$\F_{\bar{\lambda}}$};
  \fill[blue] (0.6, 3.0) circle (2pt);
  \fill[brown] (1, 3.5) circle (2pt);
  \fill[magenta] (1.1, 2.9) circle (2pt);

  \draw[line width=1pt] (0, 0) .. controls (0, 1) .. (1, 1.2);
  \draw[line width=1pt] (1, 1.2) .. controls (1.5, 0.7) .. (0, 0);
  \draw[line width=1pt] (2, 1.3) .. controls (1, 1.5) .. (2.2, 2.1);
  \draw[line width=1pt] (2.2, 2.1) .. controls (2.5, 2) .. (2, 1.3) node[pos=1, right] {$\F_{\hat\lambda}$};
  \fill[green] (1.7, 1.6) circle (2pt);
  \draw[<->, gray, thick] (0.1, 3.9) -- (1.65, 1.65) node[pos=0.8, left] {$\mathcal{E}$};
  \draw[<->, olive, dashed, thick] (1.2,2.7) -- (1.8, 1.9) node[pos=0.3, right] {$\emcm$};
  \draw[line width=1pt] (2.5,2.5) .. controls (2.9, 3.5) .. (3.5, 3.9);
  \draw[line width=1pt] (3.5, 3.9) .. controls (4.2, 3) .. (2.5,2.5);
  \draw[line width=1pt] (3, 0.6) .. controls (2.8, 1.8) .. (3.2, 2);
  \draw[line width=1pt] (3.2,2) .. controls (4,1) .. (3.0, 0.6);
  \end{tikzpicture}
  \caption{Model mis-specification error}
  \label{fig:mcm}
\end{subfigure}
~
\begin{subfigure}{0.3\textwidth}
  \centering
  \begin{tikzpicture}[scale=0.9]
  \fill[purple] (2., 2.2) circle (4pt) node[right] {$\gintilde_{m,\hat\lambda}$};
  \fill[magenta] (1.0, 1.2) circle (4pt) node[below] {\ \ \ \ \ \ \ \ $\ftilde_{n,\hat\lambda}$};
  \draw[<->, BrickRed!50, dashed, thick] (1.1, 1.3) -- (1.9, 2.1) node[pos=0.3, right] {$\ehin$};
  \draw[line width=1pt] (2.2, 0.) .. controls (0.0, 1.0) .. (2.6, 4.);
  \draw[line width=1pt] (2.6, 4.) .. controls (3.9, 3.5) .. (2.2, 0.) node[pos=0.999, right] {$\F_{\hat\lambda}$};
  \end{tikzpicture}
  \caption{Hold-in error}
  \label{fig:hin}
\end{subfigure}
\caption{{\em Decompositions of excess risk $\er$.} Figure~\ref{fig:ed-orig}
  shows the decomposition of $\er$ incurred by the approximate empirical risk
  minimizer $\ftilde_{n,\lambda} \in \F_\lambda$ with respect to the Bayes
  optimal $\fstar$. Figure~\ref{fig:mcm} visualizes the additional excess risk
  in the form of the {\em model class mis-specification risk}\, $\emcm$ incurred
  by selecting the sub-optimal hyperparameter $\hat{\lambda}$ in place of
  $\bar{\lambda}$. Figure~\ref{fig:hin} visualizes the {\em hold-in risk}\,
  $\ehin$ incurred from using $\gintilde_{m,\hat\lambda}$ instead of
  $\ftilde_{n,\hat\lambda}$.}
\label{fig:edcomp}
\end{figure*}
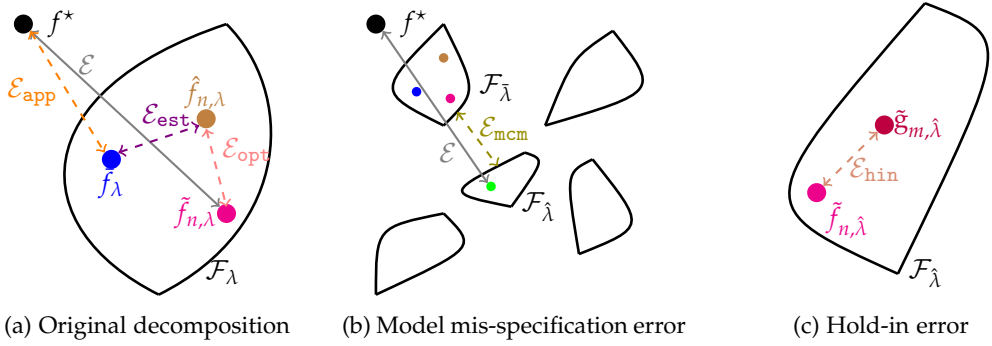
\section{Decomposing the excess risk}\label{sec:probset}
For a particular method (decision trees, linear models, neural networks), let
$\F_\lambda$ denote the function class for some {\em fixed} hyperparameter
$\lambda \in \Lambda$ (tree depth, number of trees for tree ensembles;
regularization parameter for linear and nonlinear models) in the space of valid
hyperparameters (HPs) $\Lambda$. For any model or function $f : \mathcal{X} \to
\mathcal{Y}$ with $(x, y), x \in \mathcal X, y \in \mathcal Y$ generated from a
distribution $P$ over $\mathcal{X} \times \mathcal{Y}$, and a loss function
$\ell: \mathcal{Y} \times \mathcal{Y}$, the expected risk $E(f)$ and the
empirical risk $E_n(f)$ with $n$ samples $\{(x_i, y_i)\}_{i=1}^n \sim P^n$ of
$f$ is given by
{
\begin{equation}\label{eq:risk}
\begin{split}
  & E(f)
  =  \E_{(x, y) \sim P} \left[ \ell(y, f(x)) \right]
  = \int \ell(y, f(x)) dP,
\\
  & E_n(f)
  =  \E_n \left[ \ell(y, f(x) ) \right]
  = \frac{1}{n} \sum\nolimits_{i=1}^n \ell(y_i, f(x_i)).
\end{split}
\end{equation}
}
We denote the {\em Bayes optimal model} as $\fstar$ where, for any $(x, y) \sim
P$,
\begin{equation}\label{eq:fstar}
\fstar(x) = \argmin_{\hat{y} \in \mathcal{Y}} \E\left [ \ell(y,\hat{y}) | x \right].
\end{equation}
We denote with the following:
\begin{equation}\label{eq:fbar-fhat}
  \fbar_\lambda = \argmin_{f \in \F_\lambda} E(f),
  \qquad \qquad
  \fhat_{n,\lambda} = \argmin_{f \in \F_\lambda} E_n(f),
\end{equation}
as the {\em true risk minimizer} and the {\em empirical risk minimizer} (with
$n$ samples) in model class $\F_\lambda$ respectively.
Table~\ref{tab:symbs} defines the various symbols used in the sequel.

\begin{table}[tb]
\caption{Table of symbols}
\label{tab:symbs}
\begin{center}
{
\begin{tabular}{cl} 
\toprule
Symbol & Description (1st location in text) \\
\midrule
$E(f)$ & True risk of any model $f$ \eqref{eq:risk} \\
$E_n(f)$ & Empirical risk of any model $f$ with $n$ samples \eqref{eq:risk} \\
$\Lambda$ & Set of $L$ hyperparameters (HPs) $\lambda$, $L = |\Lambda|$ (\S \ref{sec:probset}) \\
$\F_\lambda$ & Model class for hyperparameter (HP) $\lambda$ (\S \ref{sec:probset}) \\
$\fstar$ & Bayes optimal predictor \eqref{eq:fstar} \\
$\fbar_{\lambda}$ & True risk minimizer in $\F_\lambda$ \eqref{eq:fbar-fhat} \\
$\fhat_{n, \lambda}$ & Empirical risk minimizer in $\F_\lambda$ with $n$ samples \eqref{eq:fbar-fhat} \\
$\ftilde_{n, \lambda}$ & Approx. empirical risk minimizer $\F_\lambda$ with $n$ samples \eqref{eq:ftilde} \\
$\bar{\lambda}$ & Oracle hyperparameter (HP) $\arg \min_{\lambda \in \Lambda} E(\fhat_{n, \lambda})$ (\S \ref{sec:probset})
\\
$\hat{\lambda}$ & Solution to empirical hyperparameter selection \eqref{eq:hpo-cv}\\
$\ginhat_{m, \lambda}$ & Empirical risk minimizer in $\F_\lambda$ with $m$ samples \eqref{eq:hpo-cv} \\
$\gintilde_{m, \lambda}$ & Approx. empirical risk minimizer in $\F_\lambda$ with $m$ samples \eqref{eq:hpo-cv} \\
\bottomrule
\end{tabular}
}
\end{center}
\end{table}

When performing ERM over $\F_\lambda$, the {\em excess risk} incurred $\er =
E(\fhat_{n,\lambda}) - E(\fstar)$ decomposes into two terms: (i)~the {\em
  approximation risk} $\eapp(\lambda) = E(\fbar_\lambda) - E(\fstar)$, and
(ii)~the {\em estimation risk} $\eest(n, \lambda) = E(\fhat_{n,\lambda}) -
E(\fbar)$. For limited number of samples $n$, there is a tradeoff between
$\eapp$ and $\eest$, where a larger function class $\F_\lambda$ usually reduces
$\eapp(\lambda)$ but increases $\eest(n,\lambda)$~\citep{vapnik2006estimation,
  devroye2013probabilistic}.
Bottou \& Bousquet~\citep{bottou2008tradeoffs} study the tradeoffs in a
``large-scale'' setting where the learning is compute bound (in addition to the
limited number of samples $n$). Given any computational budget $T$, they
consider the learning setting ``small-scale'' when the number of samples $n$ is
small enough to allow for the ERM to be performed to arbitrary precision. In
this case, the tradeoff is between the $\eapp$ and $\eest$ terms (as
above). They consider the large scale setting where the ERM needs to be
approximated given the computational budget and discuss the tradeoffs in the
excess risk of an approximate empirical risk minimizer $\ftilde_{n,\lambda} \in
\F_\lambda$. In addition to $\eapp$ and $\eest$, they introduce the {\em
  optimization risk} term $\eopt$ -- the excess risk incurred due to approximate
ERM -- and argue that, in compute-bound large-scale learning, approximate ERM on
all the samples $n$ can achieve better generalization than high precision ERM on
a subsample of size $n' \leq n$.  Figure~\ref{fig:ed-orig} provides a visual
representation of this excess risk decomposition.

Given a set $\Lambda$ of $L$ HPs $\lambda \in \Lambda$, and $n$ samples from the
true distribution, we wish to find the {\em oracle} HP $\bar{\lambda}$ such that
the (approximate) ERM solution $\ftilde_{n,\bar{\lambda}}$ has the best possible
excess risk -- $\bar{\lambda} = \argmin_{\lambda \in \Lambda}
E(\ftilde_{n,\lambda})$. However, in practice, with $n$ {\sc iid} (independent
and identically distributed) samples from $P$, we use cross-validation for model
selection and solve the following {\em bilevel} problem to pick the HP
$\hat\lambda$:
%
\begin{equation} \label{eq:hpo-cv}
\begin{split}
& 
  \hat{\lambda} = \argmin\nolimits_{\lambda \in \Lambda} \,
   E^v_{\mu}(\gintilde_{m,\lambda})
   \qquad\qquad\quad\ \ \ \mbox{\tt (outer)}
\\
& 
  \gintilde_{m,\lambda} \in \left \{
    \gin:
    E_{m} (\gin) \leq E_{m} (\ginhat_{m,\lambda}) + \rhoin
  \right\},
  \mbox{ \tt (inner)}
\end{split}
\end{equation}%
%
where the {\tt inner} problem is an approximate ERM on $\F_\lambda$ for each
$\lambda \in \Lambda$ with $m < n$ samples at an approximation tolerance of
$\rhoin > 0$ producing $\gintilde_{m, \lambda}$ (we use $\ginhat_{m,\lambda}$ to
denote the exact ERM solution in $\F_\lambda$ with $m$ samples), and the {\tt
  outer} problem considers an objective $E^v_{\mu}(\cdot)$ which is evaluated
using $\mu < n$ samples held-out from the ERM in the inner problem -- while
$E_{m}(\cdot)$ and $E^v_{\mu}(\cdot)$ might have the same form, the
\textcolor{blue}{$\cdot^v$} superscript highlights their difference.
Then, a final approximate ERM on $\F_{\hat\lambda}$ with all $n$ samples to
$\rhoout$ tolerance produces
\begin{equation}\label{eq:ftilde}
\ftilde_{n,\hat\lambda} \in \left\{ f \in \F_{\hat\lambda}:
E_n(f) \leq E_n(\fhat_{n,\hat\lambda}) + \rhoout \right\},
\end{equation}
with $\fhat_{n,\hat\lambda}$ denoting the exact ERM solution in
$\F_{\hat\lambda}$.
This embodies the common practice of splitting the samples into a {\em training}
and a {\em held-out validation} set (of sizes $m, \mu < n$ respectively with $m
+ \mu \leq n$). In {\em $k$-fold cross-validation}, the {\tt inner} ERM is
solved $k$ times for each HP $\lambda$ (on $k$ different sets of size $m =
n-\nicefrac{n}{k}$ each), and the {\tt outer} optimization averages the
objectives from $k$ held-out sets (of size $\mu = \nicefrac{n}{k}$) across the
$k$ learned models. In this paper, we focus on CV with a single
training-validation split, and defer $k$-fold CV to future work.

At this point, multiple choices have to be made for computational and
statistical purposes:
\begin{itemize}
\item The number of samples $m$ drives the computational cost of solving the
  {\tt inner} problem for each $\lambda \in \Lambda$ -- larger $m$ requires
  larger compute budget.
\item The approximation tolerance $\rhoin$ in the {\tt inner} ERM also drives
  the computational cost -- smaller $\rhoin$ requires larger compute budget.
\item The approximation tolerance $\rhoout$ in the final ERM over
  $\F_{\hat\lambda}$ drives the computational cost similar to $\rhoin$ but to a
  lesser extent since it is only over a single $\hat\lambda \in \Lambda$ instead
  of for each $\lambda \in \Lambda$. For this reason, $\rhoout$ is usually
  selected to be smaller\footnote{Often, for computational reasons, $m$ might be
    much less than $n$, and training the final $\ftilde_{n, \hat\lambda}$ on all
    $n$ samples to a tolerance of $\rhoin$ might be computationally infeasible,
    making $\rhoout > \rhoin$.} than $\rhoin$.
\item The function $\ftilde_{n,\hat\lambda}$ is selected over
  $\gintilde_{m,\hat\lambda}$ for statistical reasons since the former gets more
  training data.
\end{itemize}

Many of these choices are often made ad hoc or via trial and error. To the best
of our knowledge, there is no mathematically grounded way of making some of
these practical choices. Moreover, it is not clear what is precisely gained by
selecting $\ftilde_{n,\hat\lambda}$ over $\gintilde_{m,\hat\lambda}$. Existing
theoretical guarantees for CV based model selection focus on the excess risk of
$\gintilde_{m,\hat\lambda}$, while in practical HPO,
$\ftilde_{n,\hat\lambda}$ is deployed, indicating a gap between theory and
practice. In this paper, we try to bridge this gap, and in the process, provide
a practical heuristic that allows us to select between
$\gintilde_{m,\hat\lambda}$ and $\ftilde_{n,\hat\lambda}$ in a data-driven
manner. Furthermore, $\rhoin$ and $\rhoout$ provide a way to control the
computation vs excess risk tradeoff, but it is not clear how to set them to
extract computational gains without significantly increasing the excess risk.
We explicitly highlight the role of $\rhoin$ and $\rhoout$ in the excess risk
and provide practical heuristics to select $\rhoin$ and $\rhoout$ in a
data-driven manner to better control this tradeoff.
\subsection{Novel excess risk decomposition}
We first present some intuitive decompositions of the excess risk to understand
the different sources of additional risk (and gains!). After the selection of
$\hat{\lambda}$ by solving problem~\eqref{eq:hpo-cv}, existing literature
focuses on the excess risk of $\gintilde_{m,\hat\lambda}$, yet we are not aware
of any decomposition of its excess risk. We decompose this excess risk as:
{
\begin{equation}\begin{split}
\er & = E(\gintilde_{m,\hat\lambda}) - E(\fstar) \\
    & =
    \underbrace{E(\gintilde_{m,\hat\lambda}) - E(\gintilde_{m,\bar{\lambda}})}_{\emcm}
    + \underbrace{E(\gintilde_{m,\bar{\lambda}}) - E(\ginhat_{m,\bar{\lambda}})}_{\eopt}
    + \underbrace{E(\ginhat_{m, \bar{\lambda}}) - E(\fbar_{\bar{\lambda}}) }_{\eest}
    + \underbrace{E(\fbar_{\bar{\lambda}}) - E(\fstar)}_{\eapp},
    \label{eq:decomp-hpo-nmd}
\end{split}\end{equation}%
}%
where $\fbar_{\bar{\lambda}}, \ginhat_{m,\bar{\lambda}}$ and
$\gintilde_{m,\bar{\lambda}}$ are the true risk minimizer, exact ERM
solution and approximate ERM solution respectively in the function class
$\F_{\bar{\lambda}}$ corresponding to the oracle HP
$\bar{\lambda}$. We introduce a new term $\emcm =
E(\gintilde_{m,\hat\lambda}) - E(\gintilde_{m,\bar{\lambda}})$, the {\em
  model class mis-specification risk}. Figure~\ref{fig:mcm} visualizes this term
in the excess risk. This term incorporates the excess risk from selecting a
suboptimal HP (and corresponding model class).
However, there is a potential additional excess risk that is often ignored in
literature but is considered crucial in practice -- the risk from learning the
model $\gintilde_{m,\hat\lambda}$ on $m < n$ samples instead of all $n$ samples,
or the {\em hold-in risk}, defined as $\ehin = E(\gintilde_{m,\hat\lambda}) -
E(\ftilde_{n,\hat\lambda})$. This term is visualized in
Figure~\ref{fig:hin}. This excess risk term does not appear explicitly in the
risk decomposition \eqref{eq:decomp-hpo-nmd} for $\gintilde_{m,\hat\lambda}$ but
rather is implicitly incorporated in the estimation risk $\eest$.  However, when
studying the excess risk of $\ftilde_{n,\hat\lambda}$, the $\ehin$ does
explicitly appear in the decomposition:
%
\begin{equation}\begin{split}
\er & = E(\ftilde_{n,\hat\lambda}) - E(\fstar)\\
    & =
    \underbrace{E(\ftilde_{n,\hat\lambda}) - E(\gintilde_{m,\hat\lambda})}_{-\ehin}
    + \underbrace{E(\gintilde_{m,\hat\lambda}) - E(\gintilde_{m,\bar{\lambda}})}_{\emcm}
    \label{eq:decomp-hpo-md}
    + \underbrace{E(\gintilde_{m, \bar{\lambda}}) - E(\fstar) }_{
      \eopt+\eest+\eapp \text{ (see \eqref{eq:decomp-hpo-nmd})}
    }
\end{split}\end{equation}%
%
This excess risk decomposition for $\ftilde_{n,\hat\lambda}$ is different from
previous decompositions in that the ``$-\ehin$'' term in this excess risk
decomposition is potentially a {\em risk deficit} instead of an additional risk,
highlighting potential {\em risk we can recover} from this common practice of
training on all the data with the selected HP $\hat\lambda$.
These decompositions are intended to explicitly highlight the different sources
of risk (and gains) in the practical HPO process, providing some
intuition into the problem.
\subsection{Empirical Validation} \label{sec:probset:emp}
\begin{table}[tb]
\caption{Empirical (relative) estimate of $\emcm = E(\gintilde_{m,\hat\lambda})
  - E(\gintilde_{m,\bar\lambda})$ across 40 OpenML datasets of varying number of
  samples $n$ and varying sizes of the held-out validation set
  $\nicefrac{\mu}{n}$. We report the percentage of the experiments (for each
  combination of $n$ and $\nicefrac{\mu}{n}$) where $\hat \lambda \not= \bar
  \lambda$. For these experiments, we also report the (estimate of the) average
  relative excess risk $\tilde \Delta$ incurred.}
\label{tab:emcm-est}
\centering
\begin{tcolorbox}[width=0.6\textwidth, colback=yellow!5, before skip=0.0cm, after skip=0.0cm, boxsep=0.0cm, top=0.1cm, bottom=0.1cm]
\begin{center}
{
\begin{tabular}{lccccc}
  $n$  &  $\mu / n$
  & \multicolumn{2}{c}{AuROC}
  & \multicolumn{2}{c}{Acc} \\
  &  &  $\hat \lambda \not= \bar \lambda$  & $\tilde \Delta$
  & $\hat \lambda \not= \bar \lambda$  & $\tilde \Delta$ \\
  \midrule
  1000   &  0.1  & 95 &  2.00 & 75 &  1.67  \\
     &  0.2  & 65 &  0.85 & 60 &  1.24  \\
     &  0.3  & 50 &  1.12 & 50 &  1.43  \\
  \midrule
  5000   &  0.1  & 65 &  0.12 & 80 &  0.60  \\
     &  0.2  & 70 &  0.14 & 75 &  0.51  \\
     &  0.3  & 70 &  0.08 & 80 &  0.69  \\
  \midrule
  10000    &  0.1  & 65 &  0.11 & 75 &  0.24  \\
     &  0.2  & 65 &  0.10 & 85 &  0.22  \\
     &  0.3  & 50 &  0.01 & 90 &  0.27  \\
  \midrule
  50000    &  0.1  & 80 &  0.34 & 90 &  0.25  \\
     &  0.2  & 90 &  0.25 & 85 &  0.16  \\
     &  0.3  & 80 &  0.27 & 90 &  0.18  \\
\end{tabular}
}
\end{center}
\end{tcolorbox}
\end{table}
To evaluate the practical significance of these newly introduced risk terms
$\emcm$ and $\ehin$, we consider the HPO problem with
LightGBM~\citep{ke2017lightgbm} across 40 OpenML binary classification
datasets~\citep{OpenML2013}. We consider 10 datasets each with number of rows in
the ranges 1000-5000, 5000-10000, 10000-50000 and 50000-100000. For each
dataset, we consider 3 different values of $\nicefrac{\mu}{n}$ (the held-out
fraction). We perform this exercise with two classification metrics -- area
under the ROC curve (AuROC) and balanced accuracy (Acc).  We approximate the
true risks for the post-hoc analysis using an additional test set not involved
in the HPO. We detail the datasets and HP search space in
Appendix~\ref{asec:emp}.

For each HPO experiment (dataset and held-out fraction), we note whether the
selected HP $\hat \lambda$ matches the oracle HP $\bar \lambda$ (found post-hoc
using the test set), and the (relative) estimate $\tilde \Delta$ of $\emcm$. We
report the aggregate findings for each set of size range and held-out fraction
in Table~\ref{tab:emcm-est}. The results indicate that, with the smaller
datasets ($n \in$ 1000-5000), a higher value of $\nicefrac{\mu}{n}$ reduces the
chances of missing the oracle HP $\bar \lambda$, but this effect is no longer
present with the larger datasets. As the dataset sizes increase, the chances of
missing the oracle HP does increase on aggregate, but the relative risk $\tilde
\Delta$ decreases from $2\%$ down to around $0.2\%$. So the $\emcm$ term
benefits from larger data but the effect is still significant. However, there is
no explicit indication of how the different terms such as $n, \mu$ play a role.

\begin{table}[t]
\caption{Comparing relative performances of $\ftilde_{n,\hat\lambda}$ and
  $\gintilde_{m,\hat\lambda}$ in HPO of LightGBM with balanced accuracy.}
\label{tab:ehin-acc}
\centering
\begin{tcolorbox}[colback=yellow!5, width=0.65\textwidth, before skip=0.0cm, after skip=0.0cm, boxsep=0.0cm, top=0.1cm, bottom=0.1cm]
\begin{center}
{
\begin{tabular}{lcccccc}
  $n$  &  $\mu / n$
  & $\ftilde \approx \gintilde$
  & $\ftilde \checkmark$ & $\ftilde$ gain
  & $\gintilde \checkmark$ & $\gintilde$ gain \\
  \midrule
   1000	 & 0.1 &  40  & 45  &  1.25  &  15  &  0.18 \\
     & 0.2 &  25  & 60  &  1.81  &  15  &  1.77 \\
     & 0.3 &  30  & 50  &  2.57  &  20  &  1.50 \\
   \midrule
   5000	 & 0.1 &  30  & 45  &  0.51  &  25  &  0.48 \\
     & 0.2 &  20  & 55  &  0.47  &  25  &  0.66 \\
     & 0.3 &  20  & 70  &  0.82  &  10  &  0.74 \\
   \midrule
   10000	 & 0.1 &  10  & 60  &  0.17  &  30  &  0.12 \\
     & 0.2 &  5  & 70  &  0.29  &  25  &  0.09 \\
     & 0.3 &  20  & 60  &  0.54  &  20  &  0.03 \\
   \midrule
   50000	 & 0.1 &  15  & 55  &  0.16  &  30  &  0.07 \\
     & 0.2 &  15  & 70  &  0.10  &  15  &  0.09 \\
     & 0.3 &  0  & 80  &  0.16  &  20  &  0.03 \\
\end{tabular}
}
\end{center}
\end{tcolorbox}
\end{table}

To understand the impact of $\ehin$, we further compare the performance of the
$\gintilde_{m, \hat \lambda}$ involved in the HPO to the final retrained
$\ftilde_{n, \hat \lambda}$ in the above HPO experiments. We note the percentage
of the time (i)~their performances were within a relative difference of
$10^{-5}$ ($\ftilde \approx \gintilde$), (ii)~$\ftilde_{n, \hat\lambda}$ was
better ($\ftilde \checkmark$), and (iii)~$\gintilde_{m, \hat\lambda}$ was better
($\gintilde \checkmark$). In both cases (ii) and (iii), we noted the average
relative gain the better choice provided (``$\ftilde$ gain'' and ``$\gintilde$
gain''). The results are aggregated across dataset sizes and held-out fraction
$\nicefrac{\mu}{n}$ in Table~\ref{tab:ehin-acc} for the balanced accuracy metric
and in Table~\ref{tab:ehin-auroc} for the AuROC metric.  The results indicate
that, in most cases, $\ftilde_{n, \hat\lambda}$ is a better choice, justifying
the common practice. However, it also indicates that, in a significant fraction
of the cases (around 20\% in most but up to 40\%), $\gintilde_{m, \hat \lambda}$
appears to be the better choice against common intuition. The results also
indicate that, when $\ftilde_{n, \hat\lambda}$ is the better choice, it also
provides higher relative gains over $\gintilde_{m, \hat\lambda}$ on average
across most experimental settings. But the average relative gains of
$\gintilde_{m, \hat\lambda}$ over $\ftilde_{n, \hat\lambda}$ are still
significant in most cases across both classification metrics.
\begin{table}[t]
\caption{Comparing relative performances of $\ftilde_{n,\hat\lambda}$ and
  $\gintilde_{m,\hat\lambda}$ in HPO of LightGBM with AuROC.}
\label{tab:ehin-auroc}
\centering
\begin{tcolorbox}[colback=yellow!5, width=0.65\textwidth, before skip=0.0cm, after skip=0.0cm, boxsep=0.0cm, top=0.1cm, bottom=0.1cm]
\begin{center}
{
\begin{tabular}{lcccccc}
  $n$  &  $\mu / n$
  & $\ftilde \approx \gintilde$
  & $\ftilde \checkmark$ & $\ftilde$ gain
  & $\gintilde \checkmark$ & $\gintilde$ gain \\
  \midrule
  1000	 & 0.1 &  5  & 80  &  0.79  &  15  &  0.72 \\
     & 0.2 &  0  & 75  &  0.42  &  25  &  2.15 \\
     & 0.3 &  10  & 70  &  0.34  &  20  &  0.27 \\
  \midrule
  5000	 & 0.1 &  15  & 45  &  0.15  &  40  &  0.15 \\
     & 0.2 &  10  & 60  &  0.24  &  30  &  0.06 \\
     & 0.3 &  10  & 50  &  0.39  &  40  &  0.05 \\
  \midrule
  10000	 & 0.1 &  5  & 75  &  0.10  &  20  &  0.02 \\
     & 0.2 &  0  & 90  &  0.09  &  10  &  0.01 \\
     & 0.3 &  0  & 100  &  0.21  &  0  &  0.00 \\
  \midrule
  50000	 & 0.1 &  10  & 50  &  0.15  &  40  &  0.18 \\
     & 0.2 &  5  & 80  &  0.18  &  15  &  0.05 \\
     & 0.3 &  0  & 90  &  0.35  &  10  &  0.18 \\
\end{tabular}
}
\end{center}
\end{tcolorbox}
\end{table}

These results indicate that there is no single best choice and we can obtain
improved performance if we are able to make this choice in a more
problem-dependent manner. In the next section, we theoretically bound the excess
risk to explicitly understand the impact of the different choices in the HPO and
leverage these dependencies for improved performance.
\section{Bounding the excess risk}\label{sec:newform}
In this section, we bound the excess risks of $\gintilde_{m,\hat\lambda}$ and
$\ftilde_{n,\hat\lambda}$ and try to understand any improvement
$\ftilde_{n,\hat\lambda}$ might provide and the interaction with the ERM
approximation tolerances $\rhoin$ and $\rhoout$ based on our decompositions.  We
have the following result for $\gintilde_{m,\hat\lambda}$:
\begin{theorem}
\label{thm:ehpo-bnd-nmd}
Let $P$ be a distribution over $\mathcal X \times \mathcal Y$ and let $\ell:
\mathcal Y \times \mathcal Y$, $\mathcal Y \subset \Real$, be a $B$-bounded
$\beta$-Lipschitz loss. Let $\F_\lambda$ be a class of functions $f: \mathcal X
\to \mathcal Y$ for any HP $\lambda \in \Lambda$.
Let $\hat\lambda$ be the solution of \eqref{eq:hpo-cv} over the set of $L$ HPs
$\Lambda$ with ERM on $m < n$ samples to approximation tolerance $\rhoin \geq
0$, a held-out set of size $\mu < n$, and $m + \mu \leq n$.
With probability at least $1 - \delta$ for $\delta \in (0, 1)$, the excess risk
$\er$ of $\gintilde_{m,\hat\lambda}$ in \eqref{eq:decomp-hpo-nmd} is bounded as:
\begin{equation}\label{eq:excess-risk-bnd-nmd}
\begin{split}
\er & \leq
  \min\nolimits_{\lambda \in \Lambda} \left\{
    8 \, \beta \, \mathcal R_{m}(\F_\lambda)
    + \eapp(\lambda)
  \right\}
  + \rhoin
  + 2 B \, \sqrt{\log (\nicefrac{2(L + 1)}{\delta})} \left(
    \nicefrac{2}{\sqrt{2m}} + \nicefrac{1}{\sqrt{2 \mu}}
  \right),
\end{split}
\end{equation}
where $\mathcal R_{m}(\F_\lambda)$ is the Radamacher complexity of $\F_\lambda$.
\end{theorem}
The proof is provided in Appendix~\ref{asec:proofs:nmd}. The above bound
\eqref{eq:excess-risk-bnd-nmd} highlights how the different terms affect the
excess risk bounds -- we get an excess risk bound within an additive factor of
the class that possesses the minimum combined (scaled) Radamacher complexity (a
proxy for estimation risk $\eest$~\citep{bartlett2002rademacher}) and
approximation risk $\eapp$. Note that the Radamacher complexity is with respect
to $m < n$ samples, highlighting the statistical inefficiency introduced by the
held-out data for CV. The result also indicates that a larger held-out set
(larger $\mu$) is preferable. We study this excess risk bound to identify the
effect of the choices in HPO (such as $m, \mu, \rhoin$); the best choice needs
to balance the terms in the bound -- making one term, such as $\rhoin$ much
smaller (say, by an order of magnitude) than the other terms will not improve
the excess risk significantly, but an order of magnitude larger $\rhoin$ will
have significant ill-effects.
We have the following result for $\ftilde_{n,\hat\lambda}$:
\footnote{
With further structural assumptions (such as relationships between the
variance and expectations of the functions), we can improve the
$\nicefrac{1}{\sqrt{\mu}}$ dependence to $\nicefrac{1}{\mu}$. See for example,
Theorem~\ref{athm:ehpo-bnd-nmd} in Appendix~\ref{asec:proofs}. We focus on
Theorem~\ref{thm:ehpo-bnd-nmd} since this will subsequently allow us to derive
practical heuristics that we cannot with the structural assumptions required
to get the tighter excess risk bounds.
}  
\begin{theorem}
\label{thm:ehpo-bnd-md}
Under the conditions of Theorem~\ref{thm:ehpo-bnd-nmd}, let
$\ftilde_{n,\hat\lambda} \in \F_{\hat\lambda}$ be obtained via approximate ERM
with tolerance $\rhoout \geq 0$ over all $n$ samples in \eqref{eq:ftilde}. Let
$\mathcal{I}_{n,m}^{\hat\lambda}(\rhoin, \rhoout) :=
E_n(\gintilde_{m,\hat\lambda}) - E_n (\ftilde_{n,\hat\lambda})$ be the
``empirical risk improvement'' from performing the ERM over all $n$ samples.
With probability at least $1 - \delta$ for any $\delta \in (0,1)$, the excess
risk $\er$ of $\ftilde_{n,\hat\lambda}$ in \eqref{eq:decomp-hpo-md} is bounded
as:
{
\begin{equation}\label{eq:excess-risk-bnd-md}
\begin{split}
\er \leq &
  \min_{\lambda \in \Lambda} \left\{
    8 \, \beta \, \mathcal R_{m}(\F_\lambda)
    + \eapp(\lambda)
  \right\}
  + 8 \, \beta \, \mathcal{R}_n(\F_{\hat\lambda})
  + \rhoin + B'\left(
    \nicefrac{2}{\sqrt{2n}} + \nicefrac{2}{\sqrt{2m}} + \nicefrac{1}{\sqrt{2 \mu}}
  \right) - \bar{\mathcal I}
\end{split}
\end{equation}
}
where $B' := 2B\sqrt{\log(\nicefrac{2(L+2)}{\delta})}$, and
\begin{equation*}\label{def:thm:ibar}
  \bar{\mathcal I} := \max \left\{
    \mathcal{I}_{n,m}^{\hat\lambda}(\rhoin,\rhoout),
    \mathcal{I}_{n,m}^{\hat\lambda}(0,0) - \rhoout - (\nicefrac{\mu}{n})B
  \right\},
\end{equation*}
with $\mathcal{I}_{n,m}^{\hat\lambda}(0,0) = E_n(\ginhat_{m,\hat\lambda}) -
E_n(\fhat_{n,\hat\lambda})$ denoting the empirical risk improvement if $\rhoin =
\rhoout = 0$.
\end{theorem}
The proof is provided in Appendix~\ref{asec:proofs:md}. This result indicates
that we are only able to recover the hold-in risk $\ehin$ in terms of the excess
risk if the empirical risk improvement $\mathcal{I}_{n,m}^{\hat\lambda}(\rhoin,
\rhoout)$ is relatively significant. While
$\mathcal{I}_{n,m}^{\hat\lambda}(0,0)$ is always $\geq 0$ by definition of the
ERM solution $\fhat_{n,\hat\lambda}$, the quantity
$\mathcal{I}_{n,m}^{\hat\lambda}(\rhoin, \rhoout)$ depends more closely to
$\rhoout$ and, we should make $\rhoout$ small enough to extract any gain from
retraining the final $\ftilde_{n,\hat\lambda}$ on all the $n$ samples --
$\rhoout$ should not exceed a critical point where the
$\mathcal{I}_{n,m}^{\hat\lambda}(\rhoin, \rhoout)$ term is of the same order as
the $8 \beta\mathcal{R}_n(\F_{\hat\lambda})$ and the
$B'(\nicefrac{2}{\sqrt{2n}})$ terms. We cannot compute this critical point, but
we will discuss how we can select $\rhoout$ in a data-driven way. Comparing
Theorems~\ref{thm:ehpo-bnd-md} to \ref{thm:ehpo-bnd-nmd} allows us to provide a
theoretical answer to \hyperlink{q1a}{\textsf{Q1a}}.
In HPO, a critical choice is the value of $\rhoin$, properly balancing the
computational {\em and} generalization aspects. Theorems~\ref{thm:ehpo-bnd-nmd}
and \ref{thm:ehpo-bnd-md} highlight the roles of these ERM approximation
tolerances $\rhoin$ and $\rhoout$, providing an answer for
\hyperlink{q2a}{\textsf{Q2a}}.

\begin{table}[tb]
\caption{Tradeoffs of the terms in the excess risk \eqref{eq:decomp-hpo-nmd} \&
  \eqref{eq:decomp-hpo-md} based on the various choices in HPO. We use
  $\F_{\text{\sf all}} = \cup_{\lambda \in \Lambda} \F_\lambda$ to denote the
  union of all function classes. `\ua' denotes an increase while `\da' a
  decrease. }
\label{tab:tradeoffs}
\begin{center}
{
\begin{tabular}{lccccccc}
\toprule
& $\F_{\text{\sf all}}$ \ua & $|\Lambda|$ \ua
  & $n$ \ua & $\mu$ \ua & $m$ \ua &  $\rhoin$ \ua & $\rhoout$ \ua \\
\midrule
$\eapp$ & \da & \da & & & \\
$\eest$ & \ua &       & \da & \ua & \da & \\
$\eopt$ & & & & & & \ua & \\
\midrule
$\emcm$ & & \ua & \da & \da &  &  & \\
$\ehin$ &  &  & \da & \ua & \da &  & \ua \\
\bottomrule
\end{tabular}
}
\end{center}
\end{table}

These results also allow us to conceptually understand the tradeoffs better the
different terms in the excess risk. A visualization of these tradeoffs is
presented in Table~\ref{tab:tradeoffs}, highlighting the effect of the
individual parameters of the problems ($n$, $m$, $\mu$, $|\Lambda|$, etc) on the
different terms in the excess risk. For example, HPO over a large set of HPs
$\Lambda$ will potentially reduce the approximation risk $\eapp$, but might
increase the model class mis-specification risk $\emcm$. Or increasing the
validation set size $\mu$ would reduce $\emcm$ but increasing $\mu$ implies
reduced $m$ (for fixed $n$), leading to higher estimation risk $\eest$. We
believe that understanding these tradeoffs explicitly will allow us to obtain
better generalization with HPO.

\subsection{Data-driven heuristics}
In certain situations, $\rhoin$ and $\rhoout$ are specified outside our control
(for example, with models learned with techniques {\em other than} gradient
descent like decision tree), and hence it is not clear which learned model,
$\gintilde_{m,\hat\lambda}$ or $\ftilde_{n,\hat\lambda}$, is better to
deploy. For this, we present a heuristic to make a data-driven choice between
the two. In practical scenarios, the approximation risk $\eapp$ usually
dominates the excess risk. Hence, if we assume that $\min_\lambda \{8 \beta
\mathcal{R}_{m}(\F_\lambda) + \eapp(\lambda)\}$ dominates $8 \beta
\mathcal{R}_n(\F_{\hat\lambda})$ (the latter does not have $\eapp$ term and $m <
n$) then we can compare the excess risk bounds of $\gintilde_{m,\hat\lambda}$
and $\ftilde_{n,\hat\lambda}$ and utilize the following heuristic (note that
$\log 2(L+1) \approx \log 2(L+2)$ except from really small $L$), providing a
data-driven answer to \hyperlink{q1b}{\textsf{Q1b}}:
\begin{heuristic}\label{heu:md-vs-nmd}
Based on the quantities in Theorems \ref{thm:ehpo-bnd-nmd} \&
\ref{thm:ehpo-bnd-md}
and $\delta>0$, we select $\ftilde_{n, \hat\lambda}$ if
$\mathcal{I}_{n,m}^{\hat\lambda} > 2B\sqrt{2
  \log(\nicefrac{2(L+2)}{\delta})/n}$, else we select $\gintilde_{m,
  \hat\lambda}$.
\end{heuristic}

To answer \hyperlink{q2b}{\textsf{Q2b}}, we need a way to make an informed
choice in terms of $\rhoin$ and $\rhoout$.  Intuitively, we can compare $\rhoin$
to the other {\em computable} terms in the bounds \eqref{eq:excess-risk-bnd-nmd}
and \eqref{eq:excess-risk-bnd-md}
-- if $\rhoin$ is of the order of these terms or larger, the $\eopt$ risk will
contribute significantly to the excess risk; if $\rhoin$ is an order of
magnitude smaller that this term, then the $\emcm$ risk will dominate the
$\eopt$ and any further reduction in $\rhoin$ is not beneficial. Based on this
observation, we propose another heuristic for HPO with approximate ERM
to facilitate the choice of $\rhoin$ in the inner ERM of the bilevel
problem~\eqref{eq:hpo-cv}:
\begin{heuristic} \label{heu:data-dep-rhoin}
Based on the quantities in Theorem~\ref{thm:ehpo-bnd-nmd} and a scaling
parameter $\gamma > 0$, we set
%
\begin{equation}\label{eq:data-dep-rhoin}
\rhoin = \gamma \, B
\sqrt{2\log(\nicefrac{2(L+1)}{\delta})} (\nicefrac{2}{\sqrt{m}} +
\nicefrac{1}{\sqrt{\mu}}).
\end{equation}
%
\end{heuristic}
A larger $\gamma$ will imply a more computationally efficient HPO, while a
smaller $\gamma$ will improve excess risk up until a point. As we will
demonstrate in our experiments, $\gamma \sim 0.1$ is sufficiently
small\footnote{This $\gamma$ implies that the optimization risk $\eopt$ is an
  order of magnitude smaller than atleast some other term in the excess risk.}
such that we still gain computational efficiency via the ERM approximation but
not see any adverse effect on the excess risk of $\gintilde_{m,\hat\lambda}$.

While we have a very precise way of setting $\rhoin$ given the theoretical
result, the choice of $\rhoout$ is more involved. We present this in
Appendix~\ref{asec:heu-rhoout}. Note that the choice of $\rhoout$ does not play
a significant role in the computational cost of the HPO compared to
$\rhoin$ since $\rhoin$ is involved in the training for each $\lambda \in
\Lambda$ while $\rhoout$ only influences the final training of $\ftilde_{n, \hat
  \lambda}$. For this reason, {\bf if possible}, $\rhoout$ is chosen to be
significantly smaller than $\rhoin$. Heuristics~\ref{heu:data-dep-rhoin} and
\ref{heu:data-dep-rhoout} (Appendix~\ref{asec:heu-rhoout}) are our answers to
\hyperlink{q2b}{\textsf{Q2b}}.
\subsection{Empirical evaluation}\label{sec:emp}
\begin{figure*}[t]
\centering
\begin{subfigure}{0.8\textwidth}
  \centering
  \includegraphics[width=\textwidth]{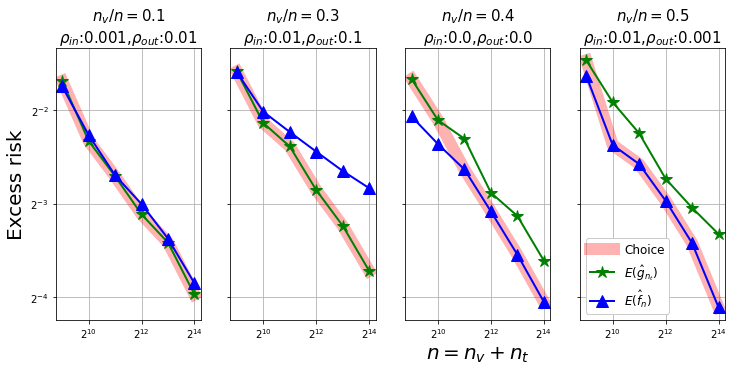}
  \caption{HPO with $|\Lambda| = 36$}
  \label{fig:choice-hpo1}
\end{subfigure}
~
\begin{subfigure}{0.8\textwidth}
  \centering
  \includegraphics[width=\textwidth]{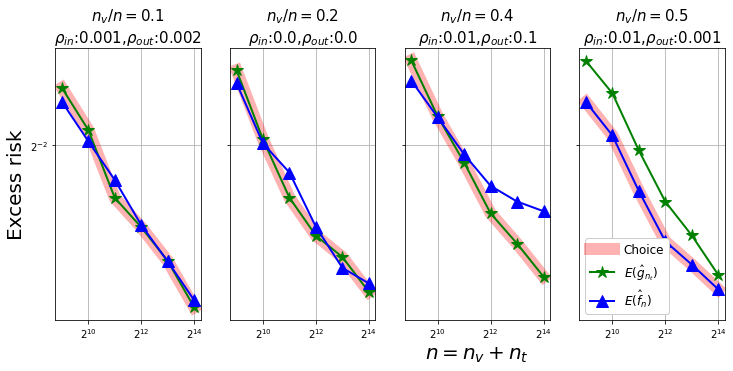}
  \caption{HPO with $|\Lambda| = 18$}
  \label{fig:choice-hpo2}
\end{subfigure}
\caption{Empirical utility of Heuristic~\ref{heu:md-vs-nmd} for data dependent
  choice of $\ftilde_{n,\hat\lambda}$ vs $\gintilde_{m,\hat\lambda}$ with a
  subset of the varying values of $n, \nicefrac{\mu}{n}, \rhoin, \rhoout$. The
  vertical axis is the excess risk ({\em lower is better}) and the horizontal
  axis is the total number of available samples $n$. Note the logscale on both
  axes.  See Appendix~\ref{asec:eval:md-v-nmd} for results on all combinations.}
\label{fig:choice}
\end{figure*}
We evaluate our heuristics on HPO with neural network configurations. We
consider neural networks to have better control over $\rhoin$ in the ERM.
We chose a synthetic data distribution to have control over the experiment and
to be able to generate fresh large samples to accurately estimate the true risks
of the different models (as opposed to our results in \S \ref{sec:probset:emp}
which were based on the true risk estimated on a limited test set). This is
common practice when empirically studying theoretical bounds on various
statistical quantities (see for example \citep{rodriguez2009sensitivity}). This
also allows us to perform the empirical evaluation under various setting (such
as different $n$, $\mu$, $\Lambda$). We set the Bayes optimal risk $E(\fstar) =
0$. We consider two HPO problems with grid-search: (a) one with 36 HP
configurations ($L = |\Lambda| = 36$), and (b) another with 18 HP configurations
($L = |\Lambda| = 18$). The data generation and the HP search spaces are
detailed in Appendix~\ref{asec:emp}.
We estimate the true risk $E(f)$ of any model $f$ with a separate large test
sample from the synthetic data distribution.  We consider sample sizes $n \in
[2^9, 2^{14}]$ and different values for $\nicefrac{\mu}{n} \in [0.1, 0.5]$ with
$m$ set to $(n - \mu)$. Each ERM involves 5 restarts and the results are
averaged over 10 trials (corresponding to different samples from the same
distribution). We set the failure probability $\delta = 0.05$.

We first evaluate the practical utility of Heuristic~\ref{heu:md-vs-nmd}, which
tries to balance the gain from utilizing the full data for obtaining
$\ftilde_{n,\hat\lambda}$ and the associated statistical cost of an additional
ERM.  Figure~\ref{fig:choice} compares this ``Choice'' based on
Heuristic~\ref{heu:md-vs-nmd} (thick \textcolor{red}{translucent red line}) to
$\ftilde_{n,\hat\lambda}$ \& $\gintilde_{m,\hat\lambda}$ (solid
\textcolor{blue}{blue $\blacktriangle$} \& \textcolor{Green}{green $\star$}
respectively) for a subset of the combinations of $\rhoin$, $\rhoout$ and
$\nicefrac{\mu}{n}$, showing the excess risk on the vertical axis as the number
of samples $n$ is increased. The results indicate that, depending on $\Lambda$,
$\rhoin$, $\rhoout$ and $\nicefrac{\mu}{n}$, $\gintilde_{m,\hat\lambda}$ might
be preferable to $\ftilde_{n,\hat\lambda}$ and vice versa -- one is not always
better than the other (as we also highlighted in \S \ref{sec:probset:emp}), and
always selecting $\ftilde_{n,\hat\lambda}$ (as done in practice) leaves room for
improvement. In both types of cases, Heuristic~\ref{heu:md-vs-nmd} is able to
select the better options in many cases
-- {\em the proposed heuristic provides a data-driven way of selecting between
  $\ftilde_{n, \hat\lambda}$ and $\gintilde_{m, \hat\lambda}$}. We provide the
full set of results for all considered values of $\rhoin$, $\rhoout$,
$\nicefrac{\mu}{n}$ for both HPO problems in Appendix~\ref{asec:eval:md-v-nmd}.
There are cases where the heuristic does not make the right choice, which
indicates that there is room for improvement.

\begin{figure*}[t]
\centering
\begin{subfigure}{0.9\textwidth}
  \centering
  \includegraphics[width=\textwidth]{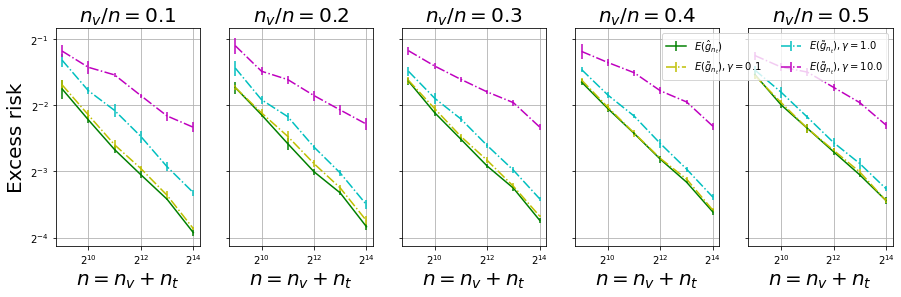}
  \caption{Excess-risk from Heuristic~\ref{heu:data-dep-rhoin}.}
  \label{fig:er-rho}
\end{subfigure}
~
\begin{subfigure}{0.9\textwidth}
  \centering
  \includegraphics[width=\textwidth]{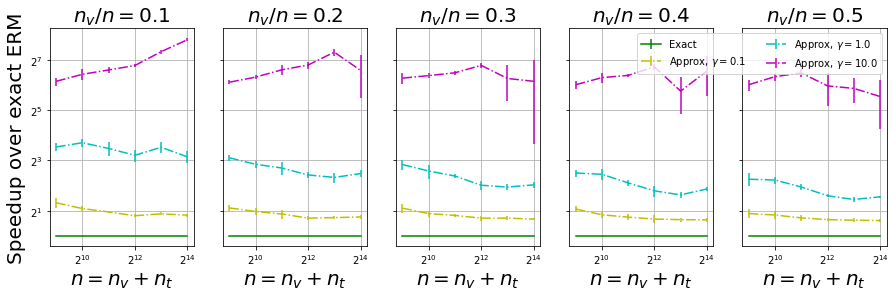}
  \caption{Speedup from Heuristic~\ref{heu:data-dep-rhoin}.}
  \label{fig:speedup-rho}
\end{subfigure}
\caption{Empirical validation of the utility of
  Heuristic~\ref{heu:data-dep-rhoin} for data-dependent choice of $\rhoin$}
\label{fig:large-scale}
\end{figure*}

To demonstrate the practical utility of the proposed
Heuristic~\ref{heu:data-dep-rhoin}, we continue with the aforementioned HP
selection problem over 36 neural network configurations on a synthetic
classification data.  We consider three choices $\gamma \in \{0.1, 1, 10\}$ in
Heuristic~\ref{heu:data-dep-rhoin}. We set $\delta = 0.05$ and consider
different values of $n$ and $\nicefrac{\mu}{n}$. Figure~\ref{fig:large-scale}
compares the performance of the HPO with exact ERM using $\ginhat_{m, \lambda},
\lambda \in \Lambda$ to the HPO with approximate ERM for different $\rhoin$
using $\gintilde_{m, \lambda}$ instead.
In Figure~\ref{fig:er-rho}, we compare the excess risk incurred from approximate
ERM with the data dependent choice of $\rhoin$ compared to exact ERM. We see
that $\gamma = 0.1$ leads to a sufficiently small $\rhoin$ that matches the
predictive performance of exact ERM. Any smaller approximation $\rhoin$ would
not improve the excess risk. The results also indicate that $\gamma = 10$ leads
to a $\rhoin$ where the optimization error dominates the excess risk, implying
that $\rhoin$ should be reduced if possible. Figure~\ref{fig:speedup-rho}
presents the computational speedups obtained for the corresponding
data-dependent choices of $\rhoin$ -- we see that we can get a $2\times$ speedup
over exact ERM without any degradation in excess risk with $\gamma = 0.1$ while
obtaining around $4 - 8\times$ speedup with slight degradation in performance
with $\gamma = 1$. These results provide empirical evidence for the practical
utility of the proposed Heuristic~\ref{heu:data-dep-rhoin} obtained from
Theorem~\ref{thm:ehpo-bnd-nmd} -- {\em the proposed heuristic provides a data
  driven way of setting the ERM approximation tolerance in HPO}.
\section{Related work}\label{sec:litreview}
While the use of smaller number of samples $m < n$ to select the HP
$\hat\lambda$ is often recognized in practice, and leads to the final model
being learned via ERM on $\F_{\hat\lambda}$ using all $n$ samples, {\em no
  theoretical guarantees exist for this process}. We explicitly study this
situation, introducing the hold-in risk $\ehin$, and provide a novel guarantee
for such a procedure in
Theorem~\ref{thm:ehpo-bnd-md}. Kearns~\citep{kearns1996bound} studies the
interaction between the approximation and estimation risks, and under certain
assumptions and restricted class of functions, proposes ways for selecting the
sizes of the training and held-out splits $m$ and $\mu$ in an informed
manner. To account for the fact that some of the training data is ``wasted'' as
the held-out set, Blum et al.~\citep{blum1999beating} propose two different ways
of retaining the Hoeffding bounds of the error estimate on the held-out set
while still being able to utilize the full training data to train models
employed at test time. These techniques are ways of modifying the standard CV.

In addition to the above, there are various theoretical analyses focusing on
various aspects of the CV process such as obtaining tight variance estimates for
the $k$-fold CV score of any given HP $\lambda \in
\Lambda$~\citep{nadeau1999inference, nadeau2003inference,
  rodriguez2009sensitivity, markatou2005analysis}. These results are
complementary to ours and could be used to extend our current results (for the
single training/validation split based CV) to $k$-fold CV, with the variance
estimates for the $k$-fold CV metrics involved in the data-driven heuristics. We
will pursue this in future research. However, note that these existing results
do not directly help us obtain tighter excess risk bounds or allow comparison
between models $\gintilde_{m, \hat\lambda}$ (used during the HPO) and the final
deployed model $\ftilde_{n, \hat\lambda}$ or provide any intuition regarding the
choices for $\rhoin$ and $\rhoout$, which are the main questions
(\hyperlink{q1a}{\textsf{Q1a}}, \hyperlink{q1b}{\textsf{Q1b}},
\hyperlink{q2a}{\textsf{Q2a}}, \hyperlink{q2b}{\textsf{Q2b}}) we study.

Finally, as we discussed in \S \ref{sec:intro}, HPO has been widely studied over
the last decade. However, the questions we focus on are complementary to any
specific HPO scheme. We do not focus on how the HP was found (with any specific
HPO scheme such Bayesian Optimization~\citep{shahriari2016taking}), but rather
on (i)~the ERM involved in the evaluation of any HP during the HPO, and (ii)~the
ERM involved in the final deployment after an HP is selected. Our
Heuristic~\ref{heu:data-dep-rhoin} allows us to speed up {\em any} HPO scheme
without any additional excess risk, and our Heuristics~\ref{heu:md-vs-nmd} and
\ref{heu:data-dep-rhoout} allow us to improve the predictive performance of the
deployed model for the HP $\hat\lambda$ selected via {\em any} HPO
scheme. ``Multi-fidelity'' HPO schemes~\citep{li2018hyperband, jamieson2016non,
  sabharwal2016selecting, klein2017fast, falkner2018bohb} significantly improve
the computational efficiency by adaptively setting either the training set size
$m < n$ or the optimization approximation $\rhoin$ on a per-HP basis instead of
using a single value of $m$ or $\rhoin$ for all $\lambda \in \Lambda$. This is
quite different from our proposed Heuristic~\ref{heu:data-dep-rhoin} and the
excess risk introduced by this adaptive strategy is not studied to the best of
our knowledge.  We wish to extend our tradeoff analysis to multi-fidelity HPO in
future work. Meta-learning is another way of improving the efficiency of the HPO
process~\citep{vanschoren2018meta}, and has been used in some AutoML
toolkits~\citep{feurer2015efficient, feurer2020auto}. Recently, some theoretical
guarantees have been established for such meta-learning based
HPO~\citep{ram2022optimality}, and we also wish to extend our tradeoff analysis
to such meta-learning based HPO.  Note that the proposed
Heuristics~\ref{heu:md-vs-nmd} and \ref{heu:data-dep-rhoout} is still beneficial
in both the above situations (multi-fidelity and meta-learning).
\section{Conclusions} \label{sec:conc}
Our contributions focus on aspects of CV based HPO -- we explore how to leverage
the different tradeoffs in the excess risk to make various practical decisions
in the HPO process in a data-driven manner. We use the novel excess risk
decomposition and theoretical analyses to answer the two questions in CV-based
HPO: (1) When is the process of training the model on all the data after the HPO
beneficial and can we choose between the two in a data driven manner?
(\hyperlink{q1a}{\textsf{Q1a}}, \hyperlink{q1b}{\textsf{Q1b}}) (2) At what level
should we set the tolerance of the optimization involved in model training
during HPO?  (\hyperlink{q2a}{\textsf{Q2a}}, \hyperlink{q2b}{\textsf{Q2b}}).
The ideas can be utilized by data science practitioners as well as by automated
data science systems.
\bibliographystyle{plainnat}
\bibliography{refs}

\begin{thebibliography}{49}
\providecommand{\natexlab}[1]{#1}
\providecommand{\url}[1]{\texttt{#1}}
\expandafter\ifx\csname urlstyle\endcsname\relax
  \providecommand{\doi}[1]{doi: #1}\else
  \providecommand{\doi}{doi: \begingroup \urlstyle{rm}\Url}\fi

\bibitem[Arlot et~al.(2010)Arlot, Celisse, et~al.]{arlot2010survey}
Sylvain Arlot, Alain Celisse, et~al.
\newblock A survey of cross-validation procedures for model selection.
\newblock \emph{Statistics surveys}, 4:\penalty0 40--79, 2010.

\bibitem[Bartlett and Mendelson(2002)]{bartlett2002rademacher}
Peter~L Bartlett and Shahar Mendelson.
\newblock Rademacher and {Gaussian} complexities: Risk bounds and structural
  results.
\newblock \emph{Journal of Machine Learning Research}, 3\penalty0
  (Nov):\penalty0 463--482, 2002.

\bibitem[Baudart et~al.(2020)Baudart, Hirzel, Kate, Ram, and
  Shinnar]{baudart2020lale}
Guillaume Baudart, Martin Hirzel, Kiran Kate, Parikshit Ram, and Avraham
  Shinnar.
\newblock Lale: Consistent automated machine learning.
\newblock In \emph{4th KDD Workshop on Automation in Machine Learning
  (AutoML@KDD)}, 2020.

\bibitem[Baudart et~al.(2021)Baudart, Hirzel, Kate, Ram, Shinnar, and
  Tsay]{baudart2021pipeline}
Guillaume Baudart, Martin Hirzel, Kiran Kate, Parikshit Ram, Avi Shinnar, and
  Jason Tsay.
\newblock Pipeline combinators for gradual automl.
\newblock \emph{Advances in Neural Information Processing Systems},
  34:\penalty0 19705--19718, 2021.

\bibitem[Bergstra and Bengio(2012)]{bergstra2012random}
J.~Bergstra and Y.~Bengio.
\newblock Random search for hyper-parameter optimization.
\newblock \emph{Journal of Machine Learning Research}, 13\penalty0
  (1):\penalty0 281--305, 2012.

\bibitem[Bergstra et~al.(2011)Bergstra, Bardenet, Bengio, and
  K{\'e}gl]{bergstra2011algorithms}
James Bergstra, R{\'e}mi Bardenet, Yoshua Bengio, and Bal{\'a}zs K{\'e}gl.
\newblock Algorithms for hyper-parameter optimization.
\newblock \emph{Advances in neural information processing systems}, 24, 2011.

\bibitem[Bergstra et~al.(2013)Bergstra, Yamins, and Cox]{bergstra2013making}
James Bergstra, Daniel Yamins, and David Cox.
\newblock Making a science of model search: Hyperparameter optimization in
  hundreds of dimensions for vision architectures.
\newblock In \emph{International conference on machine learning}, pages
  115--123. PMLR, 2013.

\bibitem[Bergstra et~al.(2015)Bergstra, Komer, Eliasmith, Yamins, and
  Cox]{bergstra2015hyperopt}
James Bergstra, Brent Komer, Chris Eliasmith, Dan Yamins, and David~D Cox.
\newblock Hyperopt: a python library for model selection and hyperparameter
  optimization.
\newblock \emph{Computational Science \& Discovery}, 8\penalty0 (1):\penalty0
  014008, 2015.

\bibitem[Blum et~al.(1999)Blum, Kalai, and Langford]{blum1999beating}
Avrim Blum, Adam Kalai, and John Langford.
\newblock Beating the hold-out: Bounds for $k$-fold and progressive
  cross-validation.
\newblock In \emph{Proceedings of the twelfth annual conference on
  Computational learning theory}, pages 203--208, 1999.

\bibitem[Bottou and Bousquet(2008)]{bottou2008tradeoffs}
L{\'e}on Bottou and Olivier Bousquet.
\newblock The tradeoffs of large scale learning.
\newblock In \emph{Advances in neural information processing systems}, pages
  161--168, 2008.

\bibitem[Boucheron et~al.(2005)Boucheron, Bousquet, and
  Lugosi]{boucheron2005theory}
St{\'e}phane Boucheron, Olivier Bousquet, and G{\'a}bor Lugosi.
\newblock Theory of classification: A survey of some recent advances.
\newblock \emph{ESAIM: probability and statistics}, 9:\penalty0 323--375, 2005.

\bibitem[Devroye et~al.(2013)Devroye, Gy{\"o}rfi, and
  Lugosi]{devroye2013probabilistic}
Luc Devroye, L{\'a}szl{\'o} Gy{\"o}rfi, and G{\'a}bor Lugosi.
\newblock \emph{A probabilistic theory of pattern recognition}, volume~31.
\newblock Springer Science \& Business Media, 2013.

\bibitem[Falkner et~al.(2018)Falkner, Klein, and Hutter]{falkner2018bohb}
Stefan Falkner, Aaron Klein, and Frank Hutter.
\newblock {BOHB}: Robust and efficient hyperparameter optimization at scale.
\newblock In \emph{International Conference on Machine Learning}, pages
  1437--1446. PMLR, 2018.

\bibitem[Feurer et~al.(2015)Feurer, Klein, Eggensperger, Springenberg, Blum,
  and Hutter]{feurer2015efficient}
Matthias Feurer, Aaron Klein, Katharina Eggensperger, Jost Springenberg, Manuel
  Blum, and Frank Hutter.
\newblock Efficient and robust automated machine learning.
\newblock In \emph{Advances in Neural Information Processing Systems}, pages
  2962--2970, 2015.

\bibitem[Feurer et~al.(2020)Feurer, Eggensperger, Falkner, Lindauer, and
  Hutter]{feurer2020auto}
Matthias Feurer, Katharina Eggensperger, Stefan Falkner, Marius Lindauer, and
  Frank Hutter.
\newblock Auto-sklearn 2.0: Hands-free {AutoML} via meta-learning.
\newblock \emph{arXiv preprint arXiv:2007.04074}, 2020.

\bibitem[Guyon(2003)]{guyon2003design}
Isabelle Guyon.
\newblock Design of experiments of the {NIPS} 2003 variable selection
  benchmark.
\newblock In \emph{NIPS 2003 workshop on feature extraction and feature
  selection}, volume 253, 2003.

\bibitem[Hirzel et~al.(2022)Hirzel, Kate, Ram, Shinnar, and
  Tsay]{hirzel2022gradual}
Martin Hirzel, Kiran Kate, Parikshit Ram, Avraham Shinnar, and Jason Tsay.
\newblock Gradual automl using lale.
\newblock In \emph{Proceedings of the 28th ACM SIGKDD Conference on Knowledge
  Discovery and Data Mining}, pages 4794--4795, 2022.

\bibitem[Hutter et~al.(2011)Hutter, Hoos, and Leyton-Brown]{smac}
Frank Hutter, Holger~H. Hoos, and Kevin Leyton-Brown.
\newblock Sequential model-based optimization for general algorithm
  configuration.
\newblock In \emph{Proceedings of the 5th International Conference on Learning
  and Intelligent Optimization}, LION'05, pages 507--523, Berlin, Heidelberg,
  2011. Springer-Verlag.

\bibitem[Jamieson and Talwalkar(2016)]{jamieson2016non}
Kevin Jamieson and Ameet Talwalkar.
\newblock Non-stochastic best arm identification and hyperparameter
  optimization.
\newblock In \emph{Artificial Intelligence and Statistics}, pages 240--248,
  2016.

\bibitem[Katz et~al.(2020)Katz, Ram, Sohrabi, and Udrea]{katz2020exploring}
Michael Katz, Parikshit Ram, Shirin Sohrabi, and Octavian Udrea.
\newblock Exploring context-free languages via planning: The case for
  automating machine learning.
\newblock In \emph{Proceedings of the International Conference on Automated
  Planning and Scheduling}, volume~30, pages 403--411, 2020.

\bibitem[Ke et~al.(2017)Ke, Meng, Finley, Wang, Chen, Ma, Ye, and
  Liu]{ke2017lightgbm}
Guolin Ke, Qi~Meng, Thomas Finley, Taifeng Wang, Wei Chen, Weidong Ma, Qiwei
  Ye, and Tie-Yan Liu.
\newblock {LightGBM}: A highly efficient gradient boosting decision tree.
\newblock \emph{Advances in neural information processing systems}, 30, 2017.

\bibitem[Kearns(1996)]{kearns1996bound}
Michael~J Kearns.
\newblock A bound on the error of cross validation using the approximation and
  estimation rates, with consequences for the training-test split.
\newblock In \emph{Advances in Neural Information Processing Systems}, pages
  183--189, 1996.

\bibitem[Kishimoto et~al.(2022)Kishimoto, Bouneffouf, Marinescu, Ram, Rawat,
  Wistuba, Palmes, and Botea]{kishimoto2022bandit}
Akihiro Kishimoto, Djallel Bouneffouf, Radu Marinescu, Parikshit Ram, Ambrish
  Rawat, Martin Wistuba, Paulito Palmes, and Adi Botea.
\newblock Bandit limited discrepancy search and application to machine learning
  pipeline optimization.
\newblock In \emph{Proceedings of the AAAI Conference on Artificial
  Intelligence}, volume~36, pages 10228--10237, 2022.

\bibitem[Klein et~al.(2017)Klein, Falkner, Bartels, Hennig, and
  Hutter]{klein2017fast}
Aaron Klein, Stefan Falkner, Simon Bartels, Philipp Hennig, and Frank Hutter.
\newblock Fast bayesian hyperparameter optimization on large datasets.
\newblock \emph{Electronic Journal of Statistics}, 11\penalty0 (2):\penalty0
  4945--4968, 2017.

\bibitem[Komer et~al.(2014)Komer, Bergstra, and Eliasmith]{komer2014hyperopt}
Brent Komer, James Bergstra, and Chris Eliasmith.
\newblock Hyperopt-sklearn: automatic hyperparameter configuration for
  scikit-learn.
\newblock Citeseer, 2014.

\bibitem[Kotthoff et~al.(2017)Kotthoff, Thornton, Hoos, Hutter, and
  Leyton-Brown]{autoweka2}
Lars Kotthoff, Chris Thornton, Holger~H. Hoos, Frank Hutter, and Kevin
  Leyton-Brown.
\newblock {Auto-WEKA 2.0}: Automatic model selection and hyperparameter
  optimization in weka.
\newblock \emph{Journal of Machine Learning Research}, 18\penalty0
  (1):\penalty0 826--830, January 2017.
\newblock ISSN 1532-4435.
\newblock URL \url{http://dl.acm.org/citation.cfm?id=3122009.3122034}.

\bibitem[Li et~al.(2018)Li, Jamieson, DeSalvo, Rostamizadeh, and
  Talwalkar]{li2018hyperband}
Lisha Li, Kevin Jamieson, Giulia DeSalvo, Afshin Rostamizadeh, and Ameet
  Talwalkar.
\newblock Hyperband: A novel bandit-based approach to hyperparameter
  optimization.
\newblock \emph{Journal of Machine Learning Research}, 18\penalty0
  (185):\penalty0 1--52, 2018.

\bibitem[Liu et~al.(2020)Liu, Ram, Vijaykeerthy, Bouneffouf, Bramble,
  Samulowitz, Wang, Conn, and Gray]{liu2020admm}
Sijia Liu, Parikshit Ram, Deepak Vijaykeerthy, Djallel Bouneffouf, Gregory
  Bramble, Horst Samulowitz, Dakuo Wang, Andrew Conn, and Alexander Gray.
\newblock An admm based framework for automl pipeline configuration.
\newblock In \emph{Proceedings of the AAAI Conference on Artificial
  Intelligence}, volume~34, pages 4892--4899, 2020.
\newblock URL \url{https://arxiv.org/abs/1905.00424v5}.

\bibitem[Marinescu et~al.(2021)Marinescu, Kishimoto, Ram, Rawat, Wistuba,
  Palmes, and Botea]{marinescu2021searching}
Radu Marinescu, Akihiro Kishimoto, Parikshit Ram, Ambrish Rawat, Martin
  Wistuba, Paulito~P Palmes, and Adi Botea.
\newblock Searching for machine learning pipelines using a context-free
  grammar.
\newblock In \emph{Proceedings of the AAAI Conference on Artificial
  Intelligence}, volume~35, pages 8902--8911, 2021.

\bibitem[Markatou et~al.(2005)Markatou, Tian, Biswas, and
  Hripcsak]{markatou2005analysis}
Marianthi Markatou, Hong Tian, Shameek Biswas, and George Hripcsak.
\newblock Analysis of variance of cross-validation estimators of the
  generalization error.
\newblock \emph{Journal of Machine Learning Research}, 6:\penalty0 1127--1168,
  2005.

\bibitem[Massart(2007)]{massart2007concentration}
Pascal Massart.
\newblock \emph{Concentration inequalities and model selection: Ecole d'Et{\'e}
  de Probabilit{\'e}s de Saint-Flour XXXIII-2003}.
\newblock Springer, 2007.

\bibitem[Nadeau and Bengio(1999)]{nadeau1999inference}
Claude Nadeau and Yoshua Bengio.
\newblock Inference for the generalization error.
\newblock \emph{Advances in neural information processing systems}, 12, 1999.

\bibitem[Nadeau and Bengio(2003)]{nadeau2003inference}
Claude Nadeau and Yoshua Bengio.
\newblock Inference for the generalization error.
\newblock \emph{Machine Learning}, 52:\penalty0 239--281, 2003.

\bibitem[Paszke et~al.(2019)Paszke, Gross, Massa, Lerer, Bradbury, Chanan,
  Killeen, Lin, Gimelshein, Antiga, et~al.]{paszke2019pytorch}
Adam Paszke, Sam Gross, Francisco Massa, Adam Lerer, James Bradbury, Gregory
  Chanan, Trevor Killeen, Zeming Lin, Natalia Gimelshein, Luca Antiga, et~al.
\newblock {PyTorch}: An imperative style, high-performance deep learning
  library.
\newblock \emph{Advances in Neural Information Processing Systems},
  32:\penalty0 8026--8037, 2019.

\bibitem[Pedregosa et~al.(2011)Pedregosa, Varoquaux, Gramfort, Michel, Thirion,
  Grisel, Blondel, Prettenhofer, Weiss, Dubourg, Vanderplas, Passos,
  Cournapeau, Brucher, Perrot, and Duchesnay]{pedregosa2011scikit}
F.~Pedregosa, G.~Varoquaux, A.~Gramfort, V.~Michel, B.~Thirion, O.~Grisel,
  M.~Blondel, P.~Prettenhofer, R.~Weiss, V.~Dubourg, J.~Vanderplas, A.~Passos,
  D.~Cournapeau, M.~Brucher, M.~Perrot, and E.~Duchesnay.
\newblock Scikit-learn: Machine learning in {P}ython.
\newblock \emph{Journal of Machine Learning Research}, 12:\penalty0 2825--2830,
  2011.

\bibitem[Rakotoarison et~al.(2019)Rakotoarison, Schoenauer, and
  Sebag]{rakotoarison2019automated}
Herilalaina Rakotoarison, Marc Schoenauer, and Michele Sebag.
\newblock Automated machine learning with {Monte-Carlo Tree Search}.
\newblock In \emph{Proceedings of the Twenty-Eighth International Joint
  Conference on Artificial Intelligence, {IJCAI-19}}, pages 3296--3303, 2019.

\bibitem[Ram(2022)]{ram2022optimality}
Parikshit Ram.
\newblock On the optimality gap of warm-started hyperparameter optimization.
\newblock In \emph{International Conference on Automated Machine Learning},
  pages 12--1. PMLR, 2022.

\bibitem[Ram et~al.(2020)Ram, Liu, Vijaykeerthi, Wang, Bouneffouf, Bramble,
  Samulowitz, and Gray]{ram2020solving}
Parikshit Ram, Sijia Liu, Deepak Vijaykeerthi, Dakuo Wang, Djallel Bouneffouf,
  Greg Bramble, Horst Samulowitz, and Alexander~G Gray.
\newblock Solving constrained {CASH} problems with {ADMM}.
\newblock In \emph{7th ICML Workshop on Automated Machine Learning (AutoML)},
  2020.

\bibitem[Ram et~al.(2021)Ram, Gray, and Samulowitz]{ram2021leveraging}
Parikshit Ram, Alexander~G Gray, and Horst Samulowitz.
\newblock Leveraging theoretical tradeoffs in hyperparameter selection for
  improved empirical performance.
\newblock In \emph{8th ICML Workshop on Automated Machine Learning (AutoML)},
  2021.

\bibitem[Rodriguez et~al.(2009)Rodriguez, Perez, and
  Lozano]{rodriguez2009sensitivity}
Juan~D Rodriguez, Aritz Perez, and Jose~A Lozano.
\newblock Sensitivity analysis of k-fold cross validation in prediction error
  estimation.
\newblock \emph{IEEE transactions on pattern analysis and machine
  intelligence}, 32\penalty0 (3):\penalty0 569--575, 2009.

\bibitem[Sabharwal et~al.(2016)Sabharwal, Samulowitz, and
  Tesauro]{sabharwal2016selecting}
Ashish Sabharwal, Horst Samulowitz, and Gerald Tesauro.
\newblock Selecting near-optimal learners via incremental data allocation.
\newblock In \emph{Thirtieth AAAI Conference on Artificial Intelligence}, 2016.

\bibitem[Shahriari et~al.(2016)Shahriari, Swersky, Wang, Adams, and
  De~Freitas]{shahriari2016taking}
B.~Shahriari, K.~Swersky, Z.~Wang, R.~P. Adams, and N.~De~Freitas.
\newblock Taking the human out of the loop: A review of bayesian optimization.
\newblock \emph{Proceedings of the IEEE}, 104\penalty0 (1):\penalty0 148--175,
  2016.

\bibitem[Snoek et~al.(2012)Snoek, Larochelle, and Adams]{snoek2012practical}
Jasper Snoek, Hugo Larochelle, and Ryan~P Adams.
\newblock Practical bayesian optimization of machine learning algorithms.
\newblock \emph{Advances in neural information processing systems}, 25, 2012.

\bibitem[Thornton et~al.(2012)Thornton, Hoos, Hutter, and
  Leyton-Brown]{autoweka1}
Chris Thornton, Holger~H. Hoos, Frank Hutter, and Kevin Leyton-Brown.
\newblock {Auto-WEKA}: Automated selection and hyper-parameter optimization of
  classification algorithms.
\newblock \emph{arXiv}, 2012.
\newblock URL \url{http://arxiv.org/abs/1208.3719}.

\bibitem[Vanschoren(2018)]{vanschoren2018meta}
Joaquin Vanschoren.
\newblock Meta-learning: A survey.
\newblock \emph{arXiv preprint arXiv:1810.03548}, 2018.

\bibitem[Vanschoren et~al.(2013)Vanschoren, van Rijn, Bischl, and
  Torgo]{OpenML2013}
Joaquin Vanschoren, Jan~N. van Rijn, Bernd Bischl, and Luis Torgo.
\newblock {OpenML}: Networked science in machine learning.
\newblock \emph{SIGKDD Explorations}, 15\penalty0 (2):\penalty0 49--60, 2013.
\newblock \doi{10.1145/2641190.2641198}.
\newblock URL \url{http://doi.acm.org/10.1145/2641190.2641198}.

\bibitem[Vapnik(2006)]{vapnik2006estimation}
Vladimir Vapnik.
\newblock \emph{Estimation of dependences based on empirical data}.
\newblock Springer Science \& Business Media, 2006.

\bibitem[Zhou et~al.(2021)Zhou, Ram, Salonidis, Baracaldo, Samulowitz, and
  Ludwig]{zhou2021flora}
Yi~Zhou, Parikshit Ram, Theodoros Salonidis, Nathalie Baracaldo, Horst
  Samulowitz, and Heiko Ludwig.
\newblock {FLoRA}: Single-shot hyper-parameter optimization for federated
  learning.
\newblock In \emph{1st NeurIPS Workshop on New Frontiers in Federated Learning
  (NFFL 2021)}, 2021.

\bibitem[Zhou et~al.(2022)Zhou, Ram, Salonidis, Baracaldo, Samulowitz, and
  Ludwig]{zhou2022single}
Yi~Zhou, Parikshit Ram, Theodoros Salonidis, Nathalie Baracaldo, Horst
  Samulowitz, and Heiko Ludwig.
\newblock Single-shot hyper-parameter optimization for federated learning: A
  general algorithm \& analysis.
\newblock \emph{arXiv preprint arXiv:2202.08338}, 2022.

\end{thebibliography}
\clearpage
%
\appendix
\section{Detailed Proofs} \label{asec:proofs}
We will make sure of this standard result:
\begin{theorem} \label{athm:max-emp-true-risk-diff-rc} (\citep{bartlett2002rademacher})
Consider conditions and notations of Theorem~\ref{thm:ehpo-bnd-nmd}.  Then for
any $\delta > 0$, with probability at least $1 - \delta$, the following is true:
{%
\begin{equation} \label{eq:max-err-disc}
\sup_{f \in \F_\lambda} \left | E_n(f) - E(f) \right |
  \leq
     4 \, \beta \, \mathcal R_n(\F_\lambda)
     + 2 \, B \, \sqrt{\nicefrac{\log (1/\delta)}{2 n}}.
\end{equation}}
where $\mathcal R_n (\F_\lambda)$ is the Radamacher complexity of $\F_\lambda$.
\end{theorem}
\subsection{Proof of Theorem~\ref{thm:ehpo-bnd-nmd}}\label{asec:proofs:nmd}
\begin{proof}(Theorem~\ref{thm:ehpo-bnd-nmd})
By definition of $\F_{\hat\lambda}$, for any $\lambda \in \Lambda$ and $\delta'
> 0$, we have the following relationships:
{
\begin{align}
E(\gintilde_{m,\hat\lambda}) & \stackrel{(a)}{\leq}
  E^v_{\mu}(\gintilde_{m,\hat\lambda})
  + B\sqrt{\log(1/\delta')/2\mu}
  \qquad \text{ w.p. } \geq 1-\delta'
  \\
  & \stackrel{(b)}{\leq}
  E^v_{\mu}(\gintilde_{m,\lambda})
  + B\sqrt{\log(1/\delta')/2\mu}
  \\
  & \stackrel{(c)}{\leq}
  E(\gintilde_{m,\lambda})
  + 2B\sqrt{\log(1/\delta')/2\mu}
  \qquad \text{ w.p. } \geq 1-\delta',
  \label{aeq:link1-1}
\end{align}}
where $(a)$ and $(c)$ are obtained from an application of Hoeffding's inequality
while $(b)$ is obtained from the definition of $\hat\lambda$ in
\eqref{eq:hpo-cv} as the minimizer of $E_{\mu}^v(\gintilde_{m,\lambda})$.  For
any $\F_\lambda, \lambda \in \Lambda$, we have the following:
{
\begin{align}
E(\gintilde_{m,\lambda}) & \stackrel{(d)}{\leq}
  E_{m}(\gintilde_{m,\lambda})
  + 4 \beta \mathcal R_{m}(\F_\lambda)
  + 2 B \sqrt{\log(1/\delta')/2m}
  \qquad \text{ w.p. } \geq 1-\delta'
  \nonumber
  \\
  & \stackrel{(e)}{\leq}
  E_{m}(\ginhat_{m,\lambda}) + \rhoin
  + 4 \beta \mathcal R_{m}(\F_\lambda)
  + 2 B \sqrt{\log(1/\delta')/2m}
  \nonumber
  \\
  & \stackrel{(f)}{\leq}
  E_{m}(\fbar_{\lambda}) + \rhoin
  + 4 \beta \mathcal R_{m}(\F_\lambda)
  + 2 B \sqrt{\log(1/\delta')/2m}
  \nonumber
  \\
  & \stackrel{(g)}{\leq}
  E(\fbar_{\lambda}) + \rhoin
  + 8 \beta \mathcal R_{m}(\F_\lambda)
  + 4 B \sqrt{\log(1/\delta')/2m}
  \qquad \text{ w.p. } \geq 1-\delta'
  \nonumber
  \\
  & \stackrel{(h)}{\leq}
  E(\fstar) + \eapp(\lambda) + \rhoin
  + 8 \beta \mathcal R_{m}(\F_\lambda)
  + 4 B \sqrt{\log(1/\delta')/2m},
  \label{aeq:link1-2}
\end{align}}%
where $(d)$ and $(g)$ are obtained by
Theorem~\ref{athm:max-emp-true-risk-diff-rc}, $(e)$ is obtained from the problem
definition \eqref{eq:hpo-cv} for the approximate ERM solution, $(f)$ is obtained
from the definition of $\ginhat_{m,\lambda}$ as the ERM solution for any class
$\F_\lambda$, and $(h)$ is a simple application of the definition of the
approximation risk $\eapp(\lambda)$ of any class $\F_\lambda$.

Since the above holds for any $\lambda \in \Lambda$, putting \eqref{aeq:link1-1}
and \eqref{aeq:link1-2} together using the union bound with $\delta' = \delta/(2
+ 2 |\Lambda|) = \delta / 2(L+1)$ and minimizing over $\lambda \in \Lambda$ gives
us the desired result in \eqref{eq:excess-risk-bnd-nmd} with a failure
probability of at most $\delta$.
\end{proof}
\subsection{Proof of Theorem~\ref{thm:ehpo-bnd-md}}\label{asec:proofs:md}
%
\begin{proof}(Theorem~\ref{thm:ehpo-bnd-md})
By definition of $\F_{\hat\lambda}$ and $\ftilde_{n,\hat\lambda} \in
\F_{\hat\lambda}$, we have the following with an application of
Theorem~\ref{athm:max-emp-true-risk-diff-rc} w.p. $\geq 1 - \delta'$:
{
\begin{align}
 E(\ftilde_{n,\hat\lambda}) \leq
  E_n(\ftilde_{n,\hat\lambda})
  + 4 \beta \mathcal R_{n}(\F_{\hat\lambda})
  + 2 B \sqrt{\log(1/\delta')/2n}
  \label{aeq:link2-1}
\end{align}}
Now from the problem definition and the definition of
$\mathcal{I}_{n,m}^{\hat\lambda}$:
{
\begin{align}
E_n(\ftilde_{n,\hat\lambda}) & =
  E_n(\gintilde_{m,\hat\lambda})
  -\mathcal{I}_{n,m}^{\hat\lambda}(\rhoin, \rhoout)
  \label{aeq:link2-2}
\end{align}
}%
However, we have an alternate relationship as follows:
{
\begin{align}
E_n(\ftilde_{n,\hat\lambda}) & \stackrel{(a)}{\leq}
  E_n(\fhat_{n,\hat\lambda}) + \rhoout \\
  & \stackrel{(b)}{=}
  E_n(\ginhat_{m,\hat\lambda})
  - \mathcal{I}_{n,m}^{\hat\lambda}(0,0)
  + \rhoout
  \\
  & \stackrel{(c)}{\leq}
  E_n(\gintilde_{m,\hat\lambda}) + \frac{\mu}{n}B
  - \mathcal{I}_{n,m}^{\hat\lambda}(0,0)
  + \rhoout,
  \label{aeq:link2-3}
\end{align}}
where $(a)$ is from the definition of the approximate ERM solution
$\ftilde_{n,\hat\lambda}$ and the exact ERM solution $\fhat_{n,\hat\lambda}$,
$(b)$ is obtained from the definition of
$\mathcal{I}_{n,m}^{\hat\lambda}(0,0)$ in the statement of the theorem, and
$(c)$ is from the definition of $\ginhat_{m,\hat\lambda}$ as the minimizer of
$E_{m}(\cdot)$ in $\F_{\hat\lambda}$ and the fact that, for any $f$, $E_n(f) =
(m/n) E_{m}(f) + (\mu/n) E_{\mu}(f)$ and the loss $\ell$ is bounded by $B$.

Putting \eqref{aeq:link2-2} and \eqref{aeq:link2-3} gives us:
{
\begin{equation*} \label{aeq:link2-4}
\begin{split}
  E_n(\ftilde_{n,\hat\lambda}) \leq
  E_n(\gintilde_{m,\hat\lambda})
  - \underbrace{
  \max \left\{
    \mathcal{I}_{n,m}^{\hat\lambda}(\rhoin, \rhoout),
    \quad
    \mathcal{I}_{n,m}^{\hat\lambda}(0,0)
    - \frac{\mu}{n}B
    - \rhoout
  \right\}
  }_{:= \bar{\mathcal I} \text{ in Theorem~\ref{thm:ehpo-bnd-md}}}.
\end{split}
\end{equation*}}

Using above in \eqref{aeq:link2-1}, we get the following relationships:
{
\begin{align}
E(\ftilde_{n,\hat\lambda}) & \leq
  E_n(\gintilde_{m,\hat\lambda}) - \bar{\mathcal I}
  + 4 \beta \mathcal R_{n}(\F_{\hat\lambda})
  + 2 B \sqrt{\log(1/\delta')/2n}
  \qquad \text{ w. p. } \geq 1 - \delta' \nonumber
  \\
  & \leq
  E(\gintilde_{m,\hat\lambda}) - \bar{\mathcal I}
  + 8 \beta \mathcal R_{n}(\F_{\hat\lambda})
  + 4 B \sqrt{\log(1/\delta')/2n}
  \label{aeq:link2-5}
  \qquad \text{ w. p. } \geq 1-2\delta', \nonumber
\end{align}}
where the last inequality is an application of
Theorem~\ref{athm:max-emp-true-risk-diff-rc} on $\gintilde_{n,\hat\lambda} \in
\F_{\hat\lambda}$.
Combining \eqref{aeq:link2-5} with \eqref{aeq:link1-1} (which holds w.p. $\geq 1
- 2\delta'$) and \eqref{aeq:link1-2} (which holds w.p. $\geq 1 - 2\delta'$ for
each $\lambda \in \Lambda$) and applying the union bound over all $\lambda \in
\Lambda$, we have
{
\begin{align}
&
E(\ftilde_{n,\hat\lambda}) - E(\gintilde_{m,\hat\lambda})
\nonumber \\
& \quad
\leq
- \bar{\mathcal I}
  + 8 \beta \mathcal R_{n}(\F_{\hat\lambda})
  + 4 B \sqrt{\nicefrac{\log(\nicefrac{1}{\delta'})}{2n}}
& = - \bar{\mathcal I}
  + 8 \beta \mathcal R_{n}(\F_{\hat\lambda})
  + B' (\nicefrac{2}{\sqrt{2n}})
  \quad \text{ w.p.} \geq 1-2\delta',
\end{align}}
We also know that for all $\lambda\not= \hat\lambda \in \Lambda$,
{
\begin{align}
E(\gintilde_{m,\hat\lambda}) -  E(\gintilde_{m,\lambda})
&  \leq 2B\sqrt{\nicefrac{\log(\nicefrac{1}{\delta'})}{2\mu}}
  = B' (\nicefrac{1}{\sqrt{2\mu}})
  \qquad \text{ w.p. } \geq 1-2\delta',
  \nonumber \\
E(\gintilde_{m,\lambda}) -  E(\fstar)
& \leq \min_{\lambda \in \Lambda} \left[ \eapp(\lambda) + \rhoin
    + 8 \beta \mathcal R_{m}(\F_\lambda) \right]
    + 4 B \sqrt{\frac{\log(\nicefrac{1}{\delta'})}{2m}}
  \nonumber \\
&  = \min_{\lambda \in \Lambda} \left[ \eapp(\lambda) + \rhoin
    + 8 \beta \mathcal R_{m}(\F_\lambda) \right]
    + B' \left(\frac{2}{\sqrt{2m}}\right) \nonumber
\end{align}}
w.p. $\geq 1 - 2L \delta'$ where we replace $2B \sqrt{\log(1/\delta')}$ with
$B'$.

Putting the above together, we get the upperbound for $\er =
E(\ftilde_{n,\hat\lambda}) - E(\fstar)$ in the statement of the claim with
probability at least $1 - (2 + 2L + 2) \delta' = 1 - 2(L+2)\delta' = 1 - \delta$
by setting $\delta' = \delta/(2(L+2)) $.
\end{proof}

\subsection{Tighter version of Theorem~\ref{thm:ehpo-bnd-nmd}.}
\begin{theorem}[Adapted from \citet{boucheron2005theory}, Thm 8.16]
\label{athm:ehpo-bnd-nmd}
Consider the conditions and notations of Theorem~\ref{thm:ehpo-bnd-nmd}. If we
further assume that there exists a non-decreasing function $w: \Real_+ \to
\Real_+$ with $w(x)/\sqrt{x}$ non-increasing for $x \in \Real_+$ such that, for
any function $f$, $\sqrt{\text{Var} \left[| f - \fstar |\right]} \leq w\left(
E(f) - E(\fstar)\right)$, then the excess risk of $\gintilde_{m,\hat\lambda}$
can be bounded from above as:
{
\begin{equation*}\label{aeq:excess-risk-bnd-nmd-2-v2}
\begin{split}
\er \leq &
  \min_{\theta \in (0, 1)} \left(1 + \theta \right)
    \left(
      \min_{\lambda \in \Lambda} \left\{
        8 \cdot \beta \cdot \mathcal R_n(\F_\lambda)
        + \eapp(\lambda)
      \right\}
      + \rhoin
    \right.\\
   & \qquad \qquad \qquad \quad \left.
      + 2 B \cdot \sqrt{\frac{2 \log((2L+1)/\delta)}{m}}
      + 4 \cdot B \cdot \log\frac{2L+1}{\delta} \left(
        \frac{\tau^*(\mu)}{\theta} + \frac{2}{3\mu}
      \right)
    \right),
\end{split}
\end{equation*}
}
where $\tau^*(n) = \min\{ x > 0: w(x) = x \sqrt{n} \}$ for any $n>0$.
\end{theorem}

\begin{proof}
For any $\lambda \in \Lambda$ and some $\delta' > 0$, we have the following from
Berstein's inequality:
\begin{align}
E(\gintilde_{m,\lambda}) - E(\fstar) & \leq
  E_{\mu}^v(\gintilde_{m,\lambda}) - E_{\mu}^v(\fstar)
  + w(E(\gintilde_{m,\lambda}) - E(\fstar)) \sqrt{
    \frac{2 \log (1/\delta')}{\mu}
  }
  + \frac{4\log (1/\delta')}{3\mu}.
\end{align}
with probability at least $1 - \delta'$.

Let $\bar{\lambda} = \arg \min_{\lambda \in \Lambda}
E(\gintilde_{m,\lambda})$. Then again applying Berstein's inequality with
$\delta' > 0$, we have
\begin{align}
E(\fstar) - E(\gintilde_{m,\bar{\lambda}}) & \leq
  E_{\mu}^v(\fstar) - E_{\mu}^v(\gintilde_{m,\bar{\lambda}})
  + w(E(\gintilde_{m,\bar{\lambda}}) - E(\fstar)) \sqrt{
    \frac{2 \log (1/\delta')}{\mu}
  }
  + \frac{4\log (1/\delta')}{3\mu}.
\end{align}
with probability at least $1 - \delta'$.

Combining the above two inequalities over all $\lambda \in \Lambda$, using the
definition of $\gintilde_{m,\hat\lambda}$ as the minimizer of
$E_{\mu}^v(\cdot)$, and the non-decreasing nature of $w(\cdot)$ gives us the
following with probability at least $1 - (L+1)\delta'$:
\begin{align}
E(\gintilde_{m,\hat\lambda}) - E(\gintilde_{m,\bar{\lambda}}) & \leq
  2 w(E(\gintilde_{m,\hat\lambda}) - E(\fstar)) \sqrt{
    \frac{2 \log (1/\delta')}{\mu}
  }
  + \frac{8\log (1/\delta')}{3\mu}.
  \label{aeq:thm3-1}
\end{align}

Given $\tau^*(\mu)$ as defined in the statement of the claim, $w(\tau^*(\mu)) =
\tau^*(\mu) \sqrt{\mu}$. Now, either $E(\gintilde_{m,\hat\lambda}) - E(\fstar) <
\tau^*(\mu)$. Or $E(\gintilde_{m,\hat\lambda}) - E(\fstar) \geq \tau^*(\mu)$,
which gives us the following:
\begin{align}
\frac{w(E(\gintilde_{m,\hat\lambda}) - E(\fstar))}{\sqrt{E(\gintilde_{m,\hat\lambda}) - E(\fstar)}}
\stackrel{(a)}{\leq} \frac{w(\tau^*(\mu))}{\sqrt{\tau^*(\mu)}}
\stackrel{(b)}{=} \sqrt{\tau^*(\mu)} \sqrt{\mu},
\end{align}
where $(a)$ comes from the assumptions that $w(x)/\sqrt{x}$ is non-increasing in
$x \in \Real_+$ and $(b)$ comes from the definition of $\tau^*(\mu)$. Using the above in \eqref{aeq:thm3-1} gives us:
{
\begin{align}
E(\gintilde_{m,\hat\lambda}) - E(\gintilde_{m,\bar{\lambda}}) & \leq
  2 \sqrt{E(\gintilde_{m,\hat\lambda}) - E(\fstar)} \sqrt{
    \frac{2 \log (1/\delta')}{\mu}
  } \sqrt{\tau^*(\mu)}
  + \frac{8\log (1/\delta')}{3\mu}
  \\
  & \leq
  \frac{\theta}{2} \left( E(\gintilde_{m,\hat\lambda}) - E(\fstar) \right)
  + \frac{8}{2 \theta} \log (1/\delta')\tau^*(\mu)
  + \frac{8\log (1/\delta')}{3\mu}.
\end{align}
}
where we utilize the fact that the arithmetic mean is greater than or equal to
the geometric mean for some $\theta \in (0,1)$.

Then we can get, w.p. $\geq 1 - (L+1) \delta'$:
{
\begin{align}
\left(E(\gintilde_{m,\hat\lambda}) - E(\fstar)\right) (1 - \theta/2) & \leq
  E(\gintilde_{m,\bar{\lambda}}) - E(\fstar)
  + 4 \log(1/\delta') \left(
    \frac{\tau^*(\mu)}{\theta}
    + \frac{2}{3\mu}
  \right).
  \label{aeq:thm3-2}
\end{align}
}
Now we have the following by definition of $\bar{\lambda}$ and $\gintilde_{m,\bar{\lambda}} \in \F_{\bar{\lambda}}$:
{
\begin{align}
E(\gintilde_{m,\bar{\lambda}}) - E(\fstar) 
  & =
  \min_{\lambda \in \Lambda} \left(
    E(\gintilde_{m,\lambda}) - E(\fstar)
  \right)
  \\
  & \stackrel{(c)}{\leq}
  \min_{\lambda \in \Lambda} \left(
    E(\fbar_{\lambda}) - E(\fstar)
    + 2 \sup_{f \in \F_\lambda} \left|
      E(f) - E_{m}(f)
    \right|  + \rhoin
  \right)
  \\
  & \stackrel{(d)}{\leq}
  \min_{\lambda \in \Lambda} \left(
    \eapp(\lambda)
    + 8 \cdot \beta \cdot \mathcal R_n(\F_\lambda)
    + 2 B \cdot \sqrt{\frac{2 \log(1/\delta')}{m}}
    + \rhoin
  \right)
  \label{aeq:thm3-3}
  \quad
  \text{ w.p. } \geq 1 -  L \delta'
\end{align}
}
where $(d)$ is obtained from the application of
Theorem~\ref{athm:max-emp-true-risk-diff-rc} on each $\lambda\in\Lambda$, and
$(c)$ is obtained as follows:
\begin{align}
E(\gintilde_{m,\lambda}) & \leq
  E_{m}(\gintilde_{m,\lambda}) + \sup_{f \in \F_\lambda}|E(f) - E_{m}(f)|
  \\
  & \leq
  E_{m}(\ginhat_{m,\lambda}) + \rhoin
  + \sup_{f \in \F_\lambda}|E(f) - E_{m}(f)|
  \\
  & \leq
  E_{m}(\fbar_{\lambda}) + \rhoin
  + \sup_{f \in \F_\lambda}|E(f) - E_{m}(f)|
  \\
  & \leq
  E(\fbar_{\lambda}) + \rhoin
  + 2 \sup_{f \in \F_\lambda}|E(f) - E_{m}(f)|.
\end{align}

Combining \eqref{aeq:thm3-2} and \eqref{aeq:thm3-3}, setting $\delta' = \delta /
(2L + 1)$ and noting that $(1 - \theta/2)^{-1} \leq (1+\theta)$ for $\theta \in
(0,1)$ gives us the statement of the claim.
\end{proof}

\section{Data-driven heuristic for $\rhoout$} \label{asec:heu-rhoout}
While we have a very precise way of setting $\rhoin$ given the theoretical
result, the choice of $\rhoout$ is somewhat more involved. Based on
Theorem~\ref{thm:ehpo-bnd-md}, we can utilize the following heuristic to set
$\rhoout$:
\begin{heuristic}\label{heu:data-dep-rhoout}
Assuming that $\rhoout$ can be iteratively reduced during the approximate ERM
for $\ftilde_{n,\hat\lambda}$, based on the terms in
Theorem~\ref{thm:ehpo-bnd-md} and a given $\rhoin$, we propose the following
iterative scheme to set $\rhoout$ with scaling parameters $\nu \in (0,1), \gamma
> 0$: For $T > 0$, we iteratively reduce $\rhoout$ as
\begin{equation}\label{eq:data-dep-rhoout}
\begin{split}
  & \rhoout^{(T+1)} \gets \nu \cdot \rhoout^{(T)}
  \quad \text{ if } \ \
  \Gamma(\rhoout^{(T-1)}) - \Gamma(\rhoout^{(T)})
  > \gamma \cdot \kappa \\
  & \text{exit approx. ERM} \quad \text{ otherwise},
\end{split}
\end{equation}
where $\Gamma(\rhoout) := \mathcal{I}_{n,m}^{\hat\lambda}(\rhoin, \rhoout)$,
$\rhoout^{(0)} \gets \rhoin$, and $\kappa := \rhoin +
B\sqrt{2\log(\nicefrac{2(L+2)}{\delta})}(\nicefrac{2}{\sqrt{n}} +
\nicefrac{2}{\sqrt{m}} + \nicefrac{1}{\sqrt{\mu}})$.
\end{heuristic}
This heuristic leverages the fact that the empirical risk improvement
$\mathcal{I}_{n,m}^{\hat\lambda}(\rhoin,\rhoout)$ will increase as $\rhoout$ is
reduced up until a point, and the excess risk of $\ftilde_{n,\hat\lambda}$ is
closely tied to this empirical risk improvement -- more improvement implies
better excess risk. Heuristic~\ref{heu:data-dep-rhoout} tries to balance any
increase in this empirical risk improvement with the other (computable) terms,
denoted as $\kappa$, in the excess risk bound in \eqref{eq:excess-risk-bnd-md}
-- we stop reducing $\rhoout$ when the increase in the empirical risk
improvement $\Gamma(\rhoout^{(T-1)}) - \Gamma(\rhoout^{(T)})$ is an order of
magnitude below $\kappa$, at which point, other terms dominate the excess
risk. Heuristics~\ref{heu:data-dep-rhoin} and \ref{heu:data-dep-rhoout} are our
answers to \hyperlink{q2b}{\textsf{Q2b}}.

Heuristic~\ref{heu:data-dep-rhoout} just presents a way to identify when
$\rhoout$ is sufficiently small in terms of statistical performance while being
able to gain computationally when we are able to approximate the ERM in an
iterative manner and progressively decrease $\rhoout$.
\begin{figure}[ht]
  \centering
  \begin{subfigure}{0.45\textwidth}
    \includegraphics[width=\textwidth]{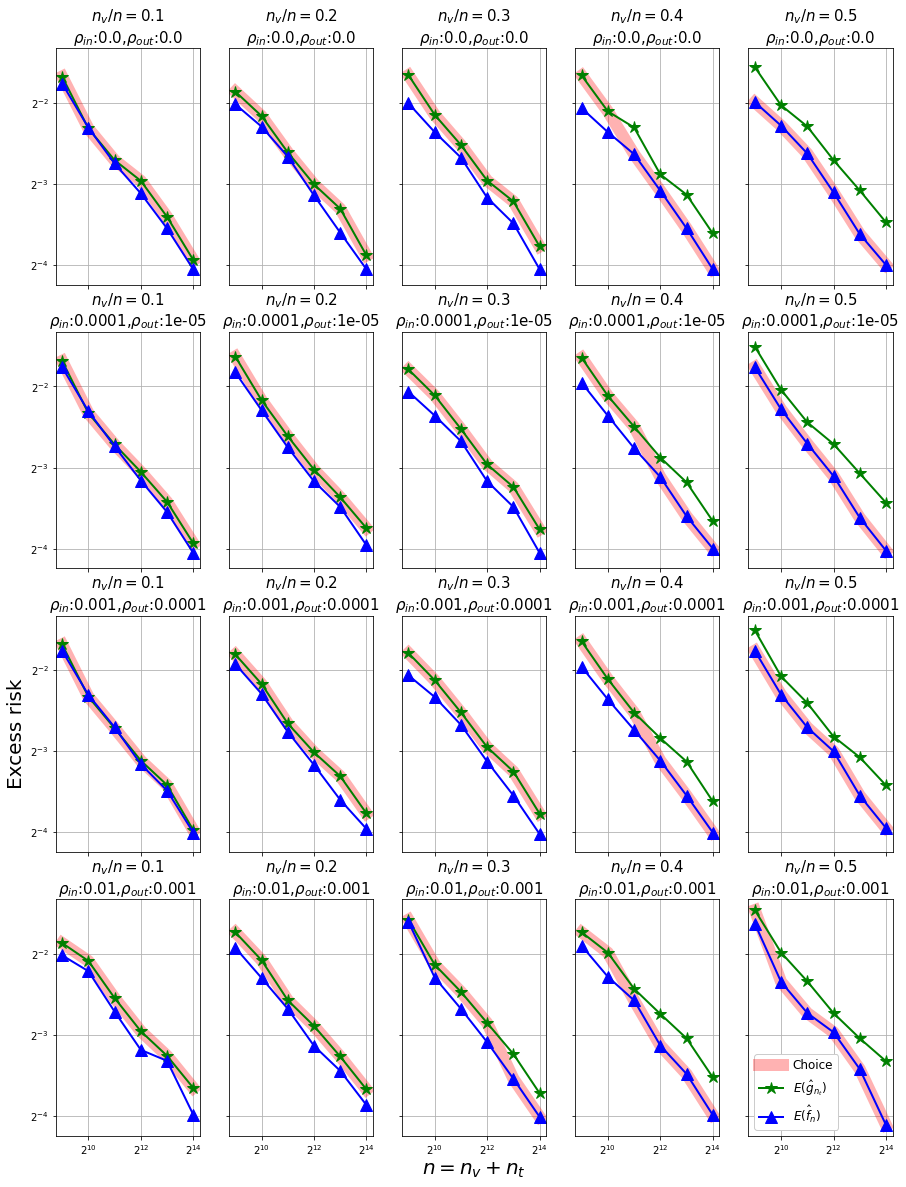}
    \caption{$\rhoout = \nicefrac{\rhoin}{10}$}
    \label{afig:h1-hpo1-1}
  \end{subfigure}
  \begin{subfigure}{0.45\textwidth}
    \includegraphics[width=\textwidth]{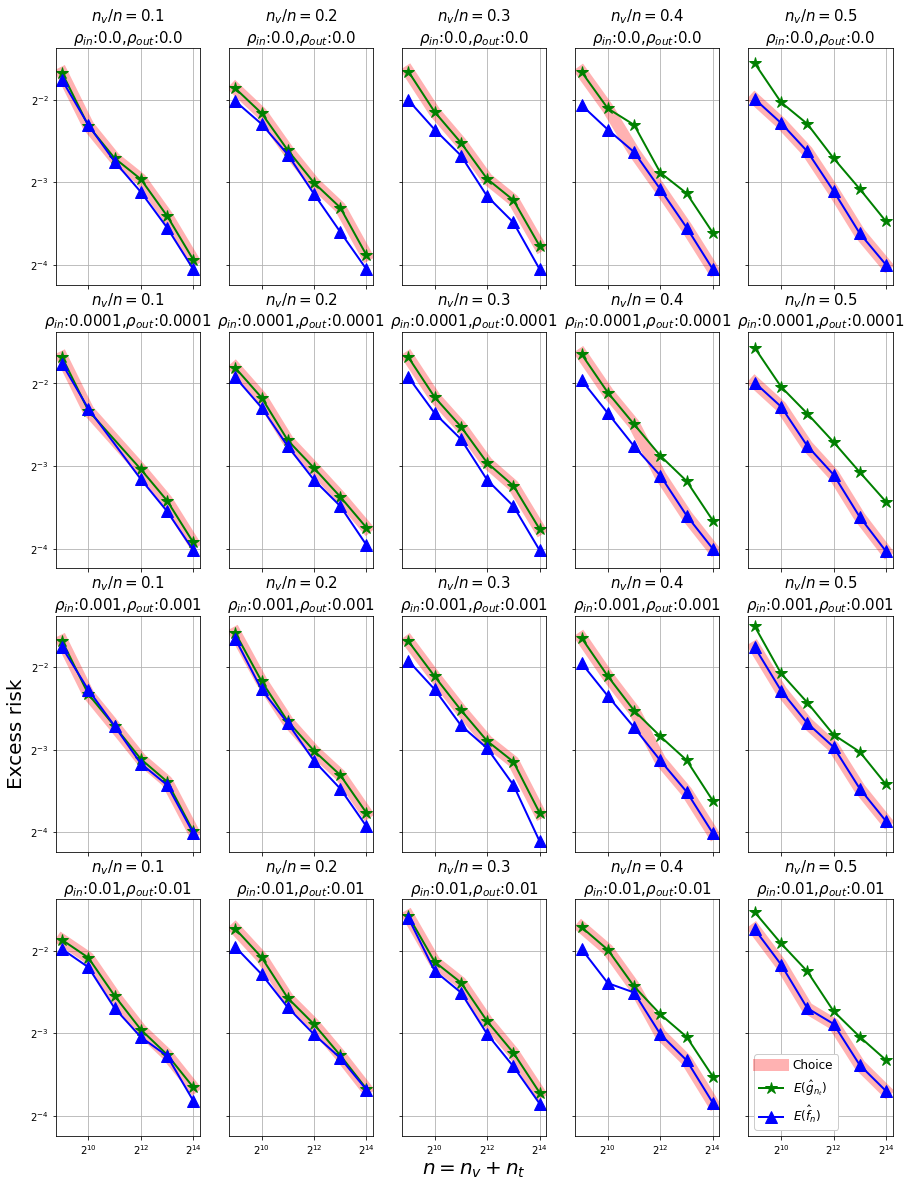}
    \caption{$\rhoout = \rhoin$}
    \label{afig:h1-hpo1-3}
  \end{subfigure}
  \begin{subfigure}{0.45\textwidth}
    \includegraphics[width=\textwidth]{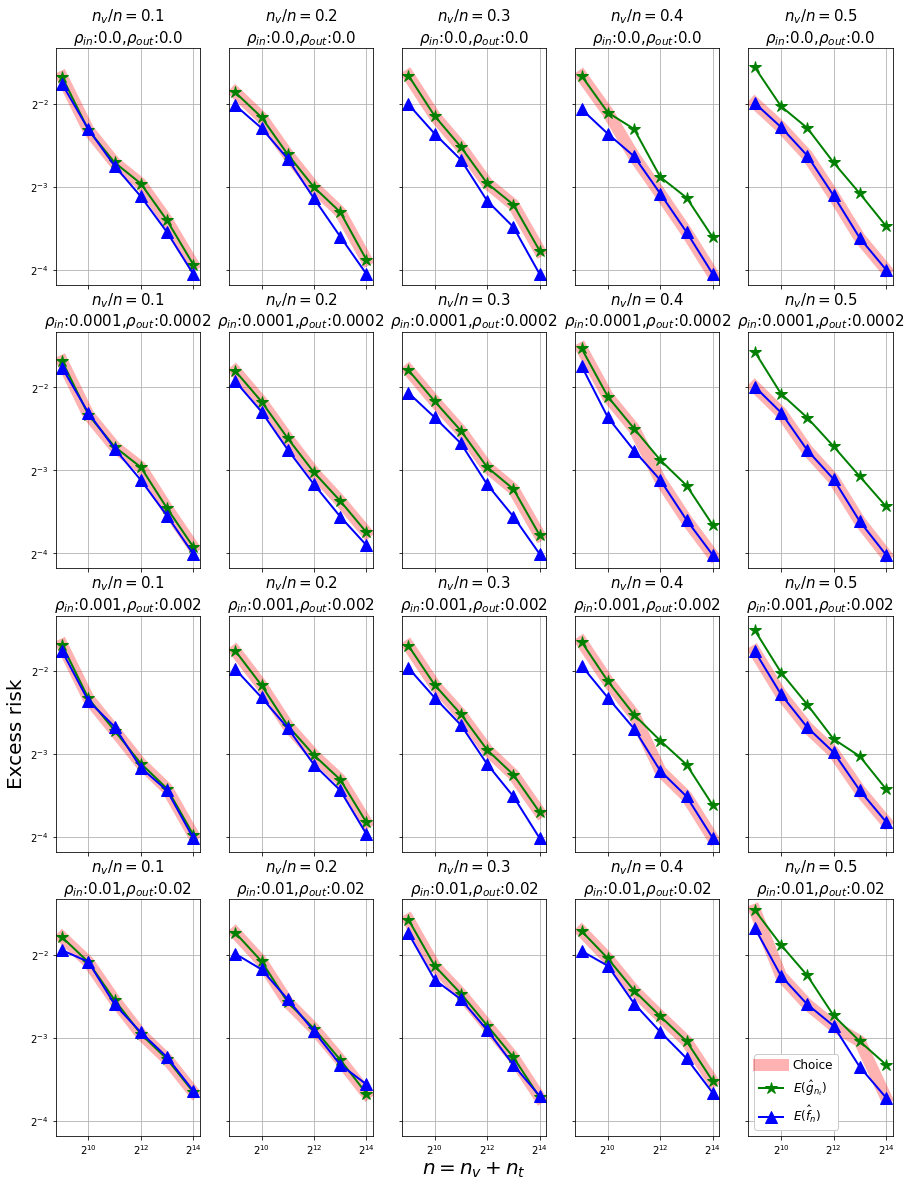}
    \caption{$\rhoout = 2 \cdot \rhoin$}
    \label{afig:h1-hpo1-4}
  \end{subfigure}
  \begin{subfigure}{0.45\textwidth}
    \includegraphics[width=\textwidth]{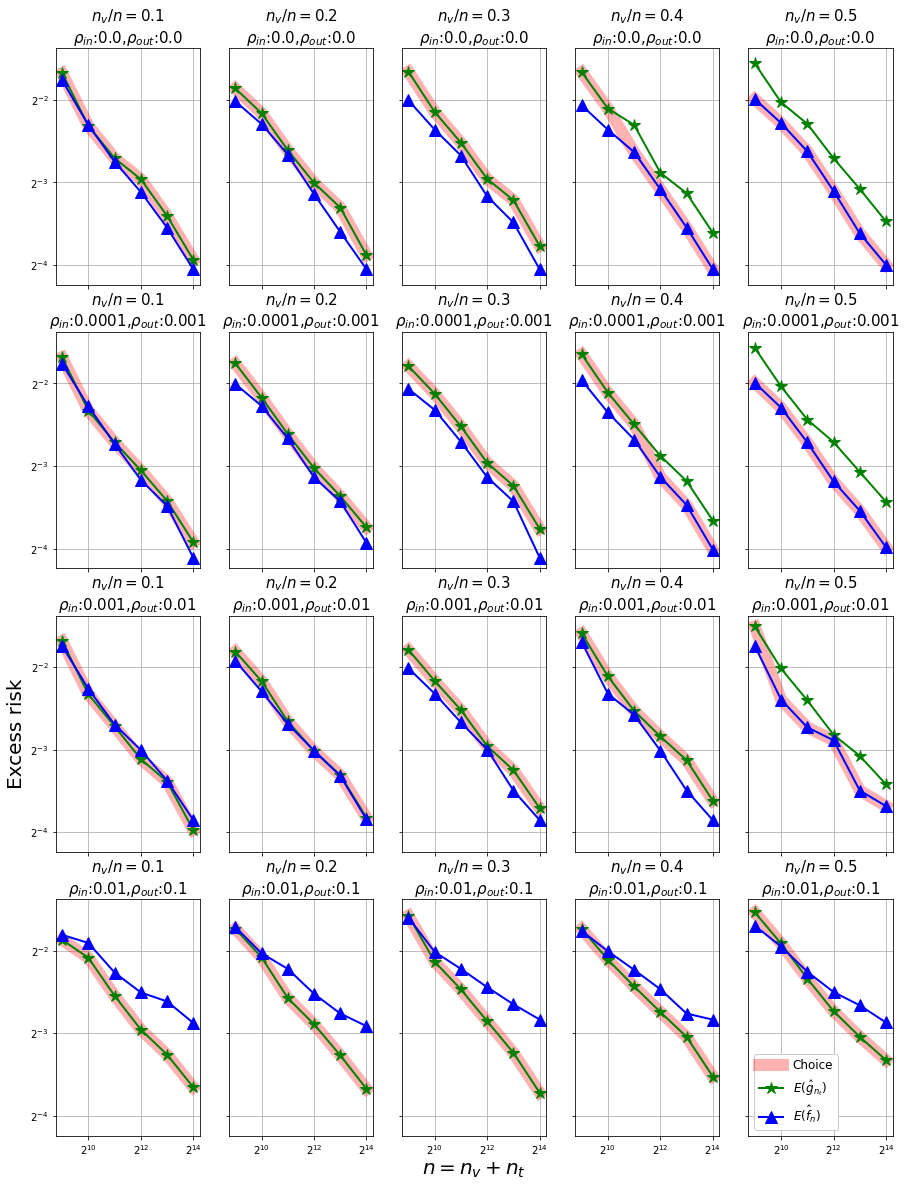}
    \caption{$\rhoout = 10 \cdot \rhoin$}
    \label{afig:h1-hpo1-5}
  \end{subfigure}
\caption{Excess-risk of data-dependent choice between $\ftilde_{n,\hat\lambda}$
  and $\gintilde_{n,\hat\lambda}$ based on Heuristic~\ref{heu:md-vs-nmd} for
  $\rhoin \in \{0.0001, 0.001, 0.01, 0.1\}$ and search space with $|\Lambda| =
  36$. Please magnify to view in detail.}
\label{afig:h1-hpo1}
\end{figure}
\section{Empirical evaluation details} \label{asec:emp}
\paragraph{Synthetic data generation.}
The binary classification data is generated using the
\texttt{make\_classification} function~\citep{guyon2003design} in
\texttt{scikit-learn}~\citep{pedregosa2011scikit}. We ensure that the classes are
not overlapping and there is no label noise, ensuring that the Bayes optimal
risk $E(\fstar) = 0$.
\paragraph{Neural network HP search space.}
We consider two different $\Lambda$s for a fully connected neural network, with
(i)~depth $\in \{1, 2, 3\}$, (ii)~number of neurons in each layer $\in \{10,
100\}$, (iii)~initial SGD learning rate $\in \{0.01, 0.1\}$, (iv)~SGD batch size
$\in \{8, 32, 128\}$. The implementation is in
\texttt{PyTorch}~\citep{paszke2019pytorch}. For one problem we consider 36
configurations (that is $L = |\Lambda| = 36$), and for another, we consider 18
configurations (that is $L = |\Lambda| = 18$).
\paragraph{LightGBM HP search space.}
We use the following search space for the{\tt LGBMClassifier} from
LightGBM~\citep{ke2017lightgbm} with {\tt RandomizedSearchCV} from {\tt
  scikit-learn}:
%
(i)~learning rate $\in [0.01,0.4]$, (ii)~number of trees $\in [100,5000]$,
(iii)~number of leaves per tree $\in [6,50]$, (iv)~minimum samples in a child
node $\in [100,500]$, (v)~minimum weight in a child node $\in [10^{-5},10^4]$,
(vi)~sub-sampling rate $\in [.1,.9]$, (vii)~maximum depth per tree $\in [-1,7]$,
(viii)~column sub-sampling rate per tree $\in [.4,.7]$, (ix)~$\alpha$
regularization $\in[0,100]$, (x)~$\lambda$ regularization $\in [0,100]$.
\paragraph{OpenML data.}
The data sets used in this experiment are listed in Table~\ref{atab:openml-list}
with their OpenML names and IDs.  For each data set, we utilize 3 different
values of the train-validation split ratio $\nicefrac{\mu}{n} \in \{0.1, 0.2,
0.3\}$.
\begin{table}[t]
\caption{Data set name \& OpenML ID of all the data sets used with 10 data sets
  each in the data size range.}
\label{atab:openml-list}
\begin{center}
{
\begin{tabular}{llll}
\toprule
\# samples & Name (OpenML ID) &  \\
\midrule
1000-4999  & kr-vs-kp (3)  & credit-g (31) \\
 & sick (38)  & spambase (44) \\
 & scene (312)  & yeast-ml8 (316) \\
 & fri-c3-1000-25 (715)  & fri-c4-1000-100 (718) \\
 & abalone (720)  & fri-c4-1000-25 (723) \\
\midrule
5000-9999  & mushroom (24)  & bank8FM (725) \\
 & cpu-small (735)  & puma32H (752) \\
 & cpu-act (761)  & delta-ailerons (803) \\
 & kin8nm (807)  & puma8NH (816) \\
 & delta-elevators (819)  & bank32nh (833) \\
\midrule
10000-49999  & BNG(tic-tac-toe) (137)  & electricity (151) \\
 & adult (179)  & BNG(breast-w) (251) \\
 & mammography (310)  & webdata-wXa (350) \\
 & pol (722)  & 2dplanes (727) \\
 & ailerons (734)  & house-16H (821) \\
\midrule
50000-99999  & vehicle-sensIT (357)  & KDDCup09-app (1111) \\
 & KDDCup09-churn (1112)  & KDDCup09-upsell (1114) \\
 & vehicleNorm (1242)  & higgs (23512) \\
 & numerai28.6 (23517)  & Run-or-walk-info (40922) \\
 & APSFailure (41138)  & kick (41162) \\
\bottomrule
\end{tabular}
}
\end{center}
\end{table}
\section{Extended empirical evaluation}\label{asec:eval}
\subsection{Further evaluation of Heuristic~\ref{heu:md-vs-nmd}}
\label{asec:eval:md-v-nmd}
In Figure~\ref{fig:choice}, we demonstrated the empirical utiity of
Heuristic~\ref{heu:md-vs-nmd} for a particular choice of $\rhoin = \rhoout = 0$
implying we leverage exact ERM in the inner level of the HPO problem (to obtain
$\ginhat_{m,\lambda},\lambda \in \Lambda$) and in the final training of the
model on the selected HP $\hat\lambda$ to obtain $\fhat_{n,\hat\lambda} \in
\F_{\hat\lambda}$. In this subsection, we evaluate Heuristic~\ref{heu:md-vs-nmd}
for other pre-set values of $\rhoin$ and $\rhoout$. We start with trying $\rhoin
\in \{0.0001, 0.001, 0.01, 0.1\}$ and then setting $\rhoout$ relative to
$\rhoin$.

We present the following results for the search space with 36 configuration in
Figure~\ref{afig:h1-hpo1}: (a)~Figure~\ref{afig:h1-hpo1-1} for $\rhoout =
\nicefrac{\rhoin}{10}$, (b)~Figure~\ref{afig:h1-hpo1-3} for $\rhoout = \rhoin$,
(c)~Figure~\ref{afig:h1-hpo1-4} for $\rhoout = 2 \cdot \rhoin$,
(d)~Figure~\ref{afig:h1-hpo1-5} for $\rhoout = 10\rhoin$.
We present the following results for the search space with 18 configuration in
Figure~\ref{afig:h1-hpo2}: (a)~Figure~\ref{afig:h1-hpo2-1} for $\rhoout =
\nicefrac{\rhoin}{10}$, (b)~Figure~\ref{afig:h1-hpo2-3} for $\rhoout = \rhoin$,
(c)~Figure~\ref{afig:h1-hpo2-4} for $\rhoout = 2 \cdot \rhoin$,
(d)~Figure~\ref{afig:h1-hpo2-5} for $\rhoout = 10\rhoin$.

For the cases where $\rhoout \leq \rhoin$ (Figures~\ref{afig:h1-hpo1-1} and
\ref{afig:h1-hpo1-3}), the excess risk of $\ftilde_{n,\hat\lambda}$ is usually
better than the excess risk of $\gintilde_{m,\hat\lambda}$, and
Heuristic~\ref{heu:md-vs-nmd}'s ``Choice'' makes the right choice when there is
significant difference between the performance of the two candidates. For the
cases where $\rhoout > \rhoin$ (Figures~\ref{afig:h1-hpo1-3} and
\ref{afig:h1-hpo1-4}), there are some situations where
$\gintilde_{m,\hat\lambda}$ has a (significantly) better excess risk over
$\ftilde_{n,\hat\lambda}$. In these cases the Heuristic~\ref{heu:md-vs-nmd}
``Choice'' is able to make the right choice -- see for example the last row in
Figure~\ref{afig:h1-hpo1-4}.
\begin{figure}[htb]
  \centering
  \begin{subfigure}{0.45\textwidth}
    \includegraphics[width=\textwidth]{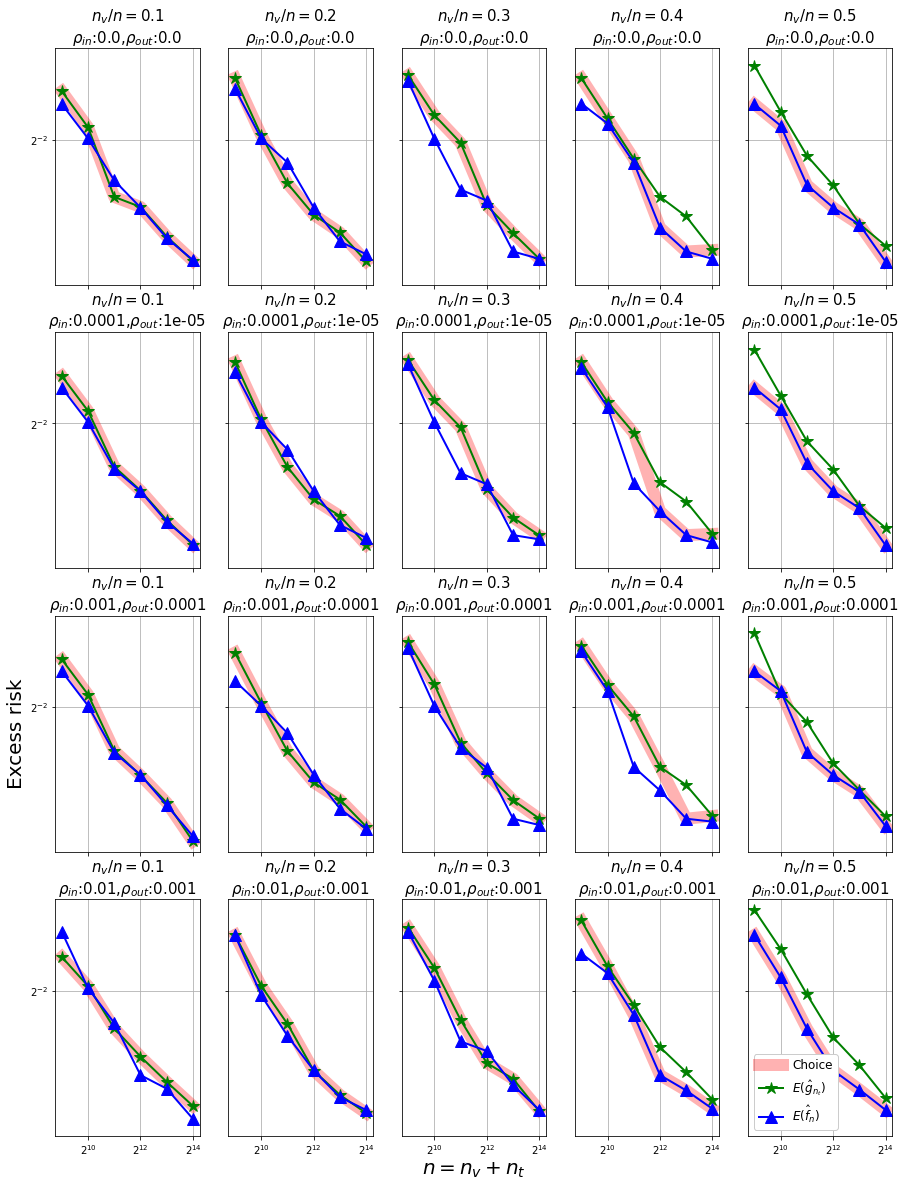}
    \caption{$\rhoout = \nicefrac{\rhoin}{10}$}
    \label{afig:h1-hpo2-1}
  \end{subfigure}
  \begin{subfigure}{0.45\textwidth}
    \includegraphics[width=\textwidth]{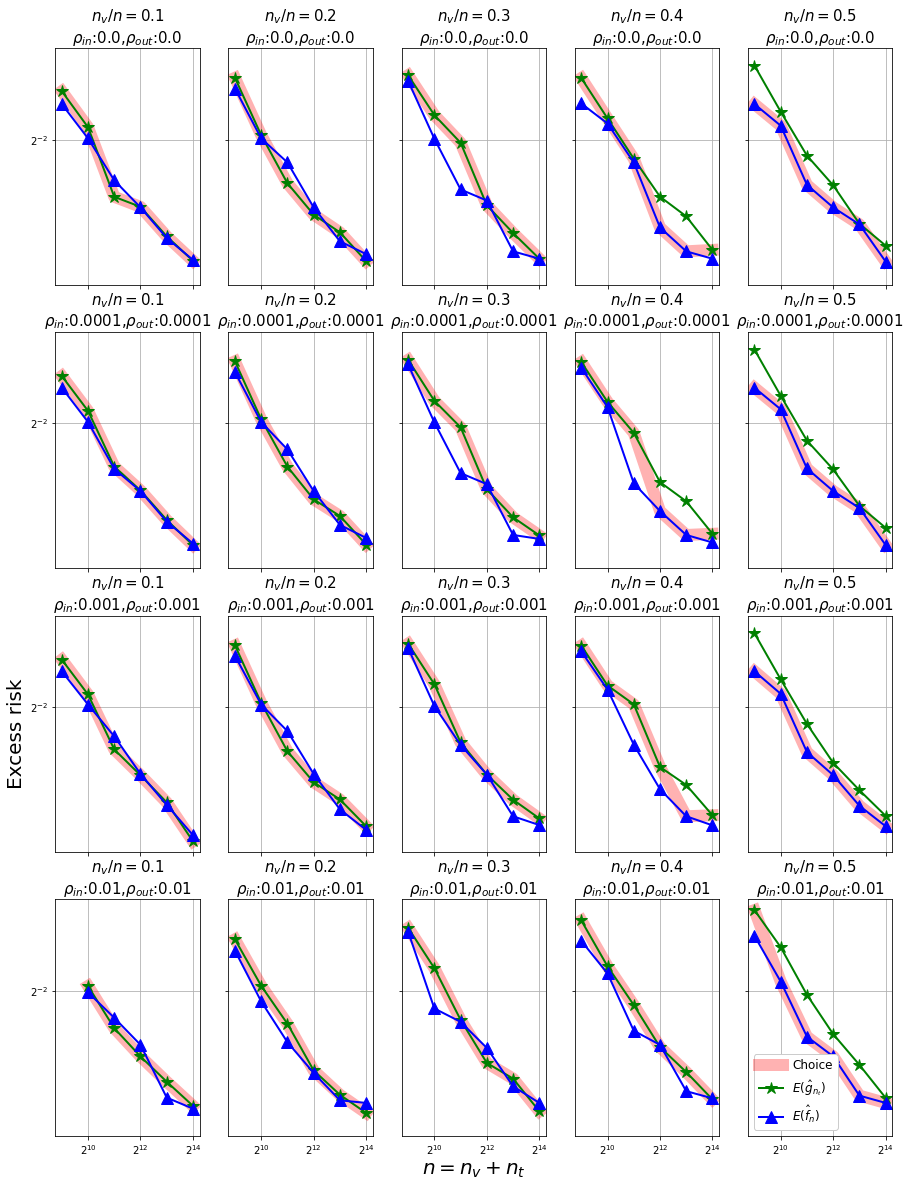}
    \caption{$\rhoout = \rhoin$}
    \label{afig:h1-hpo2-3}
  \end{subfigure}
  \begin{subfigure}{0.45\textwidth}
    \includegraphics[width=\textwidth]{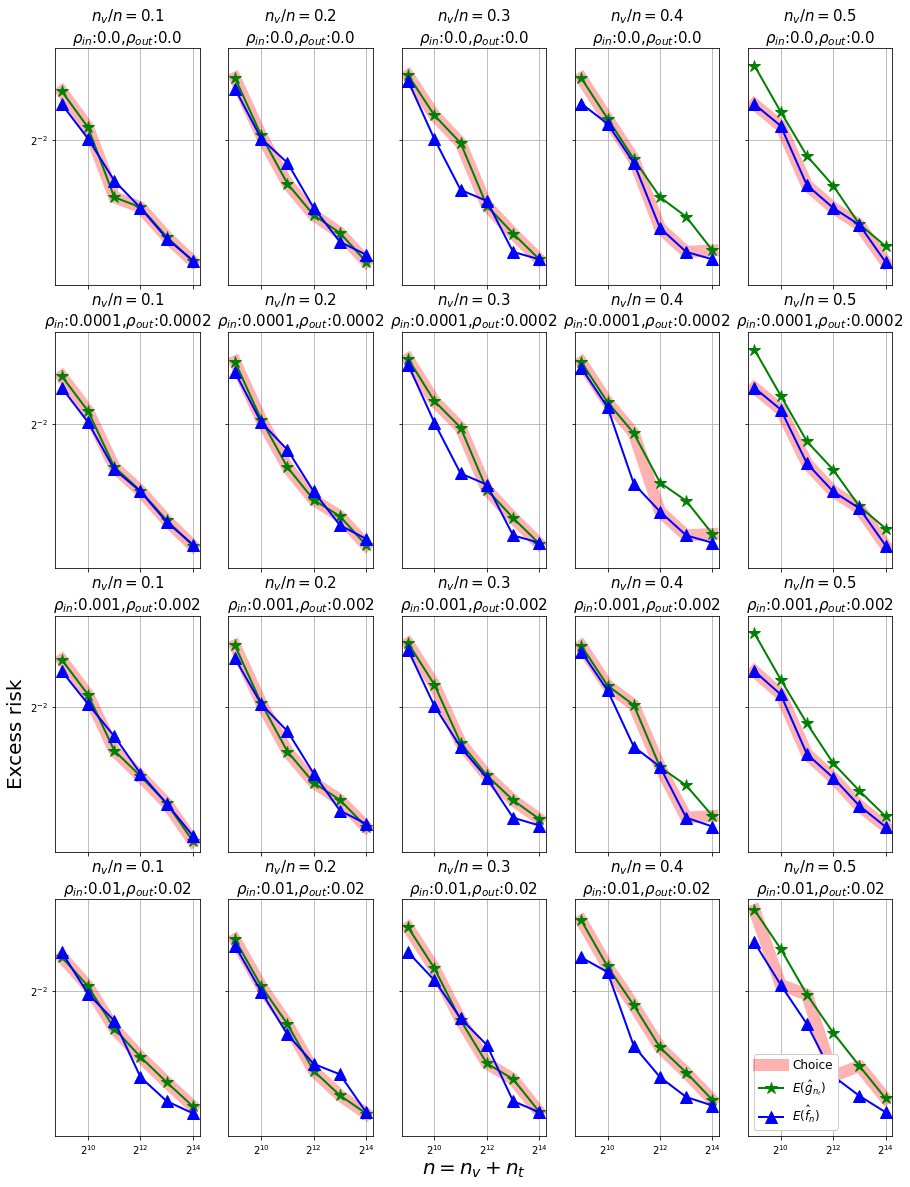}
    \caption{$\rhoout = 2 \cdot \rhoin$}
    \label{afig:h1-hpo2-4}
  \end{subfigure}
  \begin{subfigure}{0.45\textwidth}
    \includegraphics[width=\textwidth]{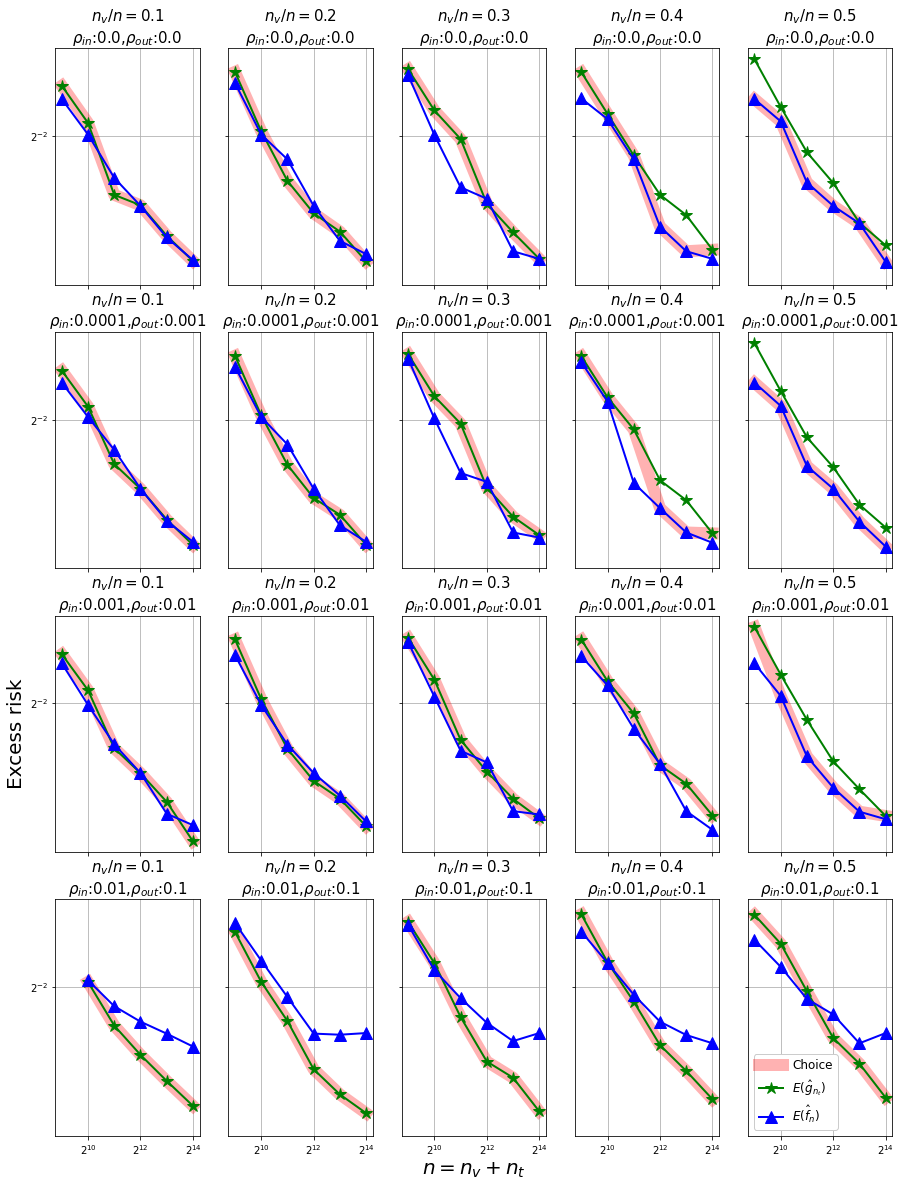}
    \caption{$\rhoout = 10 \cdot \rhoin$}
    \label{afig:h1-hpo2-5}
  \end{subfigure}
\caption{Excess-risk of data-dependent choice between $\ftilde_{n,\hat\lambda}$
  and $\gintilde_{n,\hat\lambda}$ based on Heuristic~\ref{heu:md-vs-nmd} for
  $\rhoin \in \{0.0001, 0.001, 0.01, 0.1\}$ and search space with $|\Lambda| =
  18$. Please magnify to view in detail.}
\label{afig:h1-hpo2}
\end{figure}

\subsection{Evaluation of Heuristic~\ref{heu:data-dep-rhoout}}
\label{asec:eval:rhoout}
\begin{figure}[htb]
\centering
\begin{subfigure}{0.9\textwidth}
  \centering
  \includegraphics[width=\textwidth]{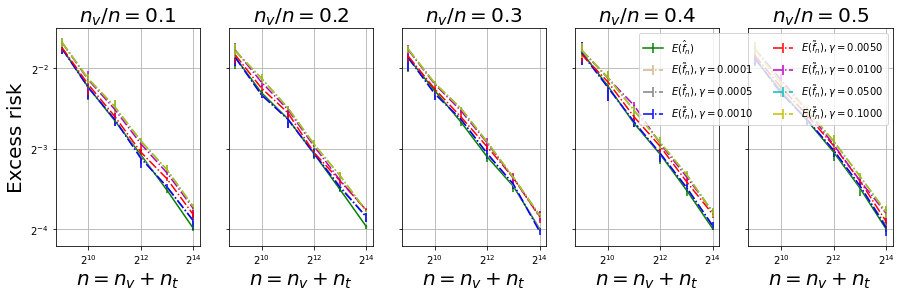}
  \caption{Excess-risk of $\ftilde$ vs. $\fhat$.}
  \label{afig:h3-er}
\end{subfigure}
\begin{subfigure}{0.9\textwidth}
  \centering
  \includegraphics[width=\textwidth]{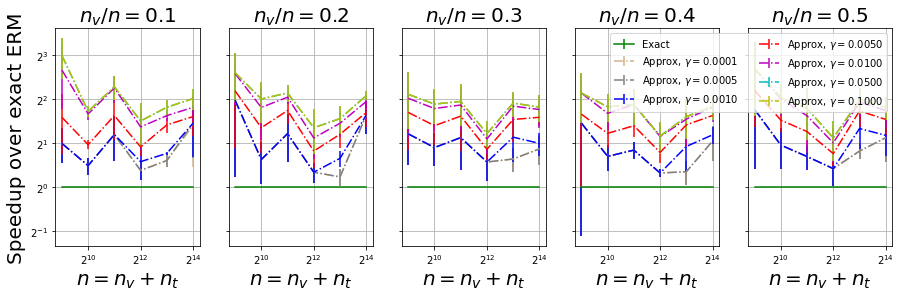}
  \caption{Speedup over exact ERM}
  \label{afig:h3-speedup}
\end{subfigure}
\caption{Empirical validation of the data-dependent choice of $\rhoout$ in
  Heuristic~\ref{heu:data-dep-rhoout}. Please magnify to view in detail.}
\label{afig:h3}
\end{figure}
Leveraging the data-dependent selection of $\rhoout$ with
Heuristic~\ref{heu:data-dep-rhoout}, we evaluate the excess risk incurred by
approximating the ERM over $\F_{\hat\lambda}$ with all $n$ samples to $\rhoout$
tolerance. We present the results for the different values of $\gamma$ in
Heuristic~\ref{heu:data-dep-rhoout} from the set $\{0.001, 0.005, 0.001, 0.005,
0.01, 0.05, 0.01\}$ in Figure~\ref{afig:h3}. The excess risks incurred and the
speedups gained from using $\ftilde_{n,\hat\lambda}$ inplace of
$\fhat_{n,\hat\lambda}$ is visualized in Figure~\ref{afig:h3} -- the solid line
corresponds to $\fhat_{n,\hat\lambda}$ while the dash-dotted lines correspond to
$\ftilde_{n,\hat\lambda}$ for different values of $\gamma$. And the results
corresponding to the excess risk in Figure~\ref{afig:h3-er} indicate that, for
$\gamma$ up to $0.005$, the increase in excess-risk is quite small. The speedups
obtained for the different choices of $\gamma$ and corresponding data-dependent
$\rhoout$ in Figure~\ref{afig:h3-speedup}. It can be seen that we can get up to
$2 \times$ speedup over exact ERM without losing much in terms of the excess
risk (see for $\gamma$ up to $0.005$); for larger values of $\gamma$ we can get
up to $4\times$ speedup if we are ready to incur some additional excess risk.
\end{document}